\documentclass{article}

\usepackage{microtype}
\usepackage{graphicx}
\usepackage{subfigure}
\usepackage{booktabs} 
\usepackage{papercommands}
\usepackage[inline]{enumitem}
\usepackage{boldline}

\usepackage{hyperref}

\newcommand{\alglinelabel}{%
  \addtocounter{ALC@line}{-1}
  \refstepcounter{ALC@line}
  \label
}

\usepackage[accepted]{icml2020}

\icmltitlerunning{Sets Clustering}

\usepackage{thm-restate}
\usepackage{color}

\makeatletter
\let\amsmath@bigm\bigm

\renewcommand{\bigm}[1]{%
  \ifcsname fenced@\string#1\endcsname
    \expandafter\@firstoftwo
  \else
    \expandafter\@secondoftwo
  \fi
  {\expandafter\amsmath@bigm\csname fenced@\string#1\endcsname}%
  {\amsmath@bigm#1}%
}

\usepackage{verbatim}

\begin{document}

\twocolumn[
\icmltitle{Sets Clustering}



\icmlsetsymbol{equal}{*}

\begin{icmlauthorlist}
\icmlauthor{Ibrahim Jubran}{equal,to}
\icmlauthor{Murad Tukan}{equal,to}
\icmlauthor{Alaa Maalouf}{equal,to}
\icmlauthor{Dan Feldman}{to}
\end{icmlauthorlist}

\icmlaffiliation{to}{Robotics \& Big Data Lab, Department of Computer Science, University of Haifa, Israel}

\icmlcorrespondingauthor{Ibrahim Jubran}{ibrahim.jub@gmail.com}

\icmlkeywords{Machine Learning, ICML}

\vskip 0.3in
]



\printAffiliationsAndNotice{\icmlEqualContribution} 

\begin{abstract}
The input to the \emph{sets-$k$-means} problem is an integer $k\geq 1$ and a set $\mathcal{P}=\br{P_1,\cdots,P_n}$ of sets in $\mathbb{R}^d$. The goal is to compute a set $C$ of $k$ centers (points) in $\mathbb{R}^d$ that minimizes the sum $\sum_{P\in \mathcal{P}} \min_{p\in P, c\in C}\left\| p-c \right\|^2$ of squared distances to these sets.
An  \emph{$\eps$-core-set} for this problem is a weighted subset of $\mathcal{P}$ that approximates this sum up to $1\pm\varepsilon$ factor, for \emph{every} set $C$ of $k$ centers in $\mathbb{R}^d$.
We prove that such a core-set of $O(\log^2{n})$ sets always exists, and can be computed in $O(n\log{n})$ time, for every input $\mathcal{P}$ and every fixed $d,k\geq 1$ and $\varepsilon \in (0,1)$. The result easily generalized for any metric space, distances to the power of $z>0$, and M-estimators that handle outliers. Applying an inefficient but optimal algorithm on this coreset allows us to obtain the first PTAS ($1+\eps$ approximation) for the sets-$k$-means problem that takes time near linear in $n$.
This is the first result even for sets-mean on the plane ($k=1$, $d=2$).
Open source code and experimental results for document classification and facility locations are also provided.

\end{abstract}

\section{Introduction} \label{sec:intro}
In machine learning it is common to represent the input as a set of $n$ points (database records) $P=\br{p_1,\cdots,p_n}$ in the Euclidean $d$-dimensional space $\REAL^d$. That is, an $n\times d$ real matrix whose rows correspond to the input points. Every point corresponds to e.g. the GPS address of a person~\cite{liao2006location,nguyen2011adaptive}, a pixel/feature in an image~\cite{tuytelaars2008local}, ``bag of words" of a document~\cite{mladenic1999text}, or a sensor's sample~\cite{dunia1996identification}. Arguably, the most common statistics of such a set is its mean (center of mass) which is the center $c\in\REAL^d$ that minimizes its sum of squared distances
$\sum_{p\in P}\Dt(p,c)=\sum_{p\in P}\norm{p-c}^2$
to the input points in $P$. Here, $\Dt(p,c):=\norm{p-c}^2$ is the squared distance between a point $p\in P$ to the center $c\in\REAL^d$.
More generally, in unsupervised learning, for a given integer (number of clusters) $k\geq1$, the \emph{$k$-means} of the set $P$ is a set $C=\br{c_1,\cdots,c_k}$ of $k$ centers (points in $\REAL^d$) that minimizes the sum of squared distances
\[
\sum_{p\in P} \Dt(p,C)=\sum_{p\in P} \min_{c\in C}\norm{p-c}^2,
\]
where $\Dt(p,C):=\min_{c\in C}\Dt(p,c)$ denotes the squared distance from each point $p\in P$ to its nearest center in $C$.
The $k$-means clustering is probably the most common clustering objective function, both in academy and industry as claimed in~\cite{hartigan1975clustering, arthur2006k, berkhin2002survey, wu2008top}.

\begin{figure}[t!]
\centering
    \subfigure[$k=1$]{
		\centering
		\includegraphics[width = 0.21\textwidth]{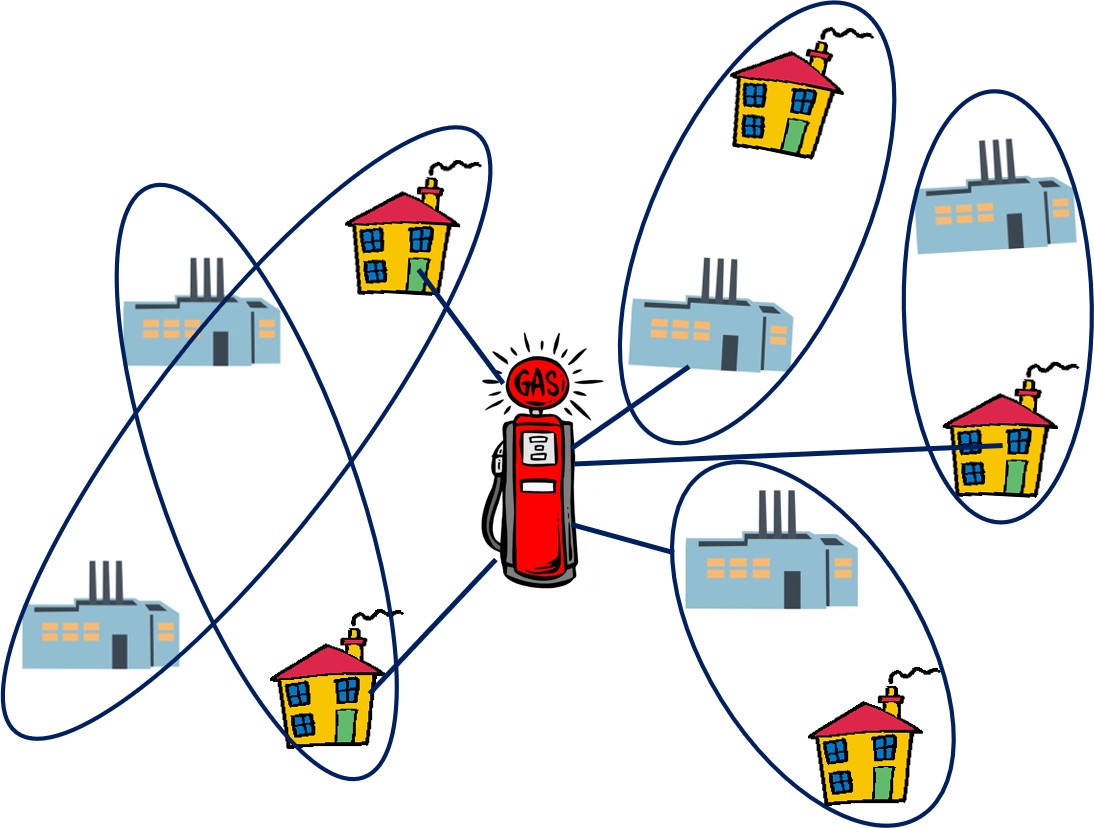}
        \label{fig:gas1}}
    \rulesep
    \subfigure[$k=2$]{
		\centering
		\includegraphics[width = 0.21\textwidth]{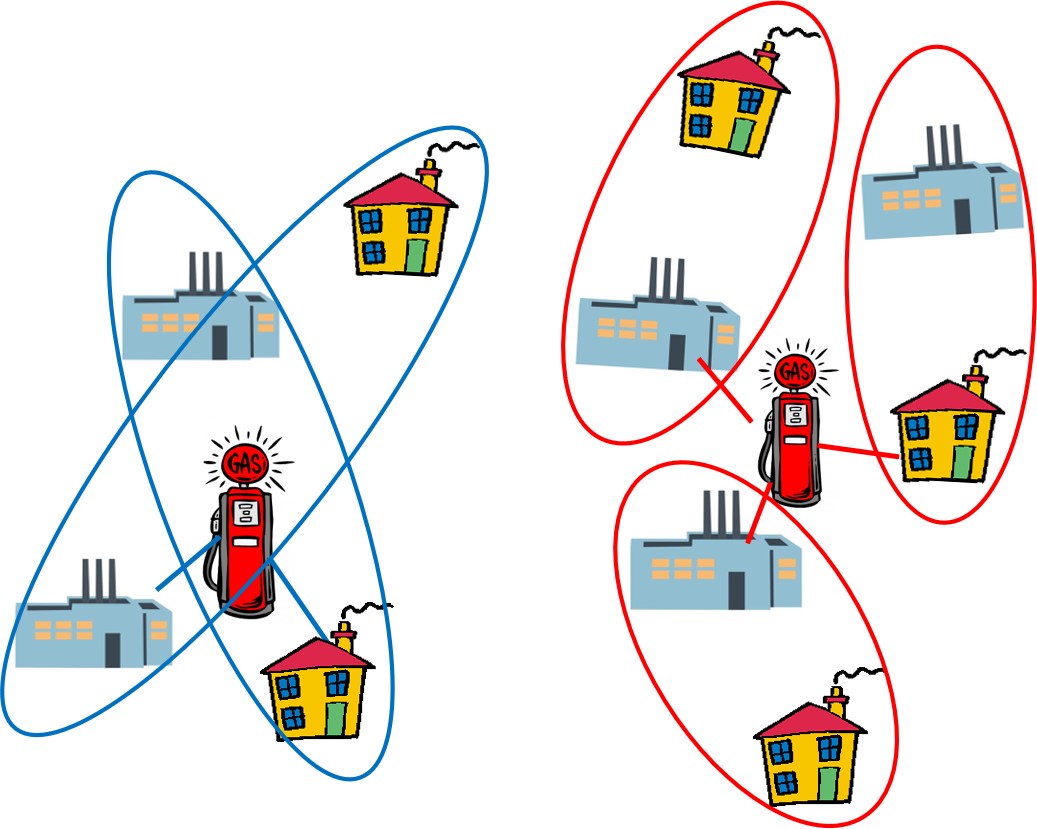}
        \label{fig:gas2}}
    \caption{\textbf{sets-$k$-means for pairs on the plane.} The input is a set of $n=5$ pairs of points that correspond to home/work addresses. (left) The distance from a gas station (in red) to a person is the smaller between its distance to its home and work address (line segments). (right) For $k=2$ gas stations, each person will choose its closest gas station; see real word database at section~\ref{sec:ER}.}
    \label{fig:facilityLoc}
\end{figure}

However, in the real-world, every database's record actually links to another database table, a GPS location may correspond to multiple GPS locations (e.g. home/work), every image consists of a set of pixels/features, every document contains a set of paragraphs, and a sensor's sample may actually be a distribution over some possible values~\cite{li2010object,li2008one,dunia1996identification,xiao2007using}. This motivates the following title and subject of this paper.

\paragraph{Sets Clustering. }Along this paper, the input is not a set of points, but rather a set $\set{P} = \br{P_1,\cdots,P_n}$ of sets in $\REAL^d$ (or any other metric space; see Section~\ref{sec:extension}), each of size $m$, denoted as $m$-sets. A natural generalization of the mean of a set $P$ is what we defined as the \emph{sets-mean} of our set $\set{P}$ of sets. The sets-mean is the point $c\in\REAL^d$ that minimizes its sum of squared distances
\begin{equation}\label{1mean}
\sum_{P\in \set{P}}\Dt(P,c)
=\sum_{P\in \set{P}}\min_{p\in P} \norm{p-c}^2,
\end{equation}
to the nearest point in each set. Here, $\Dt(P,c):=\min_{p\in P}\dist(p,c)$.

More generally, the \emph{sets-$k$-means} $C$ of $\set{P}$ is a set of $k$ points in $\REAL^d$ that minimizes its sum of squared distances
\begin{equation}\label{kmeaneq123}
\displaystyle\sum_{P \in \set{P}} \Dt(P,C)
=\sum_{P\in \set{P}}\min_{p\in P, c\in C} \norm{p-c}^2,
\end{equation}
to the nearest point in each set. Here, $\Dt(P,C):=\min_{p\in P,c\in C}\dist(p,c)$ is the closest distance between a pair in $P\times C$.

\textbf{Example. }Suppose that we want to place a gas station that will serve $n$ people whose home addresses are represented by $n$ GPS points (on the plane). The mean is a natural candidate since it minimizes the sum of squared distances from the gas station to the people; see~\cite{jubran2019introduction}. Now, suppose that the $i$th person for every $i\in\br{1,\cdots,n}=[n]$ is represented by a pair $P_i=\br{h,w}$ of points on the plane: home address $h$ and work address $w$; see Fig.~\ref{fig:facilityLoc}. It would be equally as convenient for a resident if the gas station was built next to his work address rather than his home address. Hence, the sets-mean of the addresses $\set{P}=\br{P_1,\cdots,P_n}$, as defined in the previous page, minimizes the sum of squared distances from the gas station to the nearest address of each person (either home or work). The sets-$k$-means is the set $C\subseteq\REAL^d$ of $k$ gas stations that minimizes the sum of squared Euclidean distances from each person to its nearest gas station as in~\eqref{kmeaneq123}.

\subsection{Applications}
From a theoretical point of view, sets clustering is a natural generalization of points clustering. The distance $\dist(P,C)$ between sets generalizes the distance $\dist(p,C)=\min_{c\in C}\dist(p,c)$  between a point and a set, as used e.g. in $k$-means clustering of points.

\textbf{Clustering Shapes~\cite{srivastava2005statistical}. }The first sets clustering related result appeared only recently in~\cite{marom2019k} for the special case where each of the $n$ input sets is a line (an infinite set) in $\REAL^d$. However, in this paper every input set is a finite and arbitrary set in a general metric space.

It is therefore not surprising that many of the numerous applications for points clustering can be generalized to sets clustering. Few examples are given below.

\textbf{Facility locations~\cite{cohen2019fully,blelloch2010parallel,ahmadian2013local}.}
The above gas station example immediately implies applications for Facility Location problems.

\textbf{Natural Language Processing~\cite{collobert2011natural}. }A disadvantage of the common ``bag of words" model is that the order of words in a document does not change its representation~\cite{spanakis2012exploiting}. Sets clustering can help partially overcome this issue by considering the document as the set of vectors corresponding to each of its paragraphs, as illustrated in Fig.~\ref{fig:documentClustering}.

\textbf{Hierarchical clustering~\cite{abboud2019subquadratic,murtagh1983survey}. }Here, the goal is to compute a tree of clusters. The leaves of this tree are the input points, and the next level represent their clustering into $n$ sets. In the next level, the goal is to cluster these $n$ sets into $k$ sets.

\textbf{Probabilistic databases~\cite{suciu2011probabilistic}.} Here, each data sample corresponds to a finite distribution over possible values. E.g. a sample that was obtained from a sensor with a known noise model. Algorithm for computing the minimum enclosing ball ($1$-center) for sets (distributions) was suggested in~\cite{munteanu2014smallest} using coresets, as defined in section~\ref{sec:coresets}.

\subsection{Why is it Hard? } \label{sec:why}
Computing the $k$-means of points in $\REAL^d$ ($m=1$) is already NP-hard when $k$ is not fixed, even for $d=2$.
It can be solved in $n^{O(dk)}$ time using exhaustive search as explained in~\cite{inaba1994applications}. Multiplicative $(1+\varepsilon)$ approximation is also NP-hard for constant a $\eps>0$~\cite{lee2017improved}.

For fixed $k$, deterministic constant factor approximation can be computed in time $O(ndk)$ by constructing coresets (see Section~\ref{sec:coresets}) of size $m=O(k/\varepsilon^3)$~\cite{braverman2016new,feldman2011unified}, on which the optimal exhaustive search is then applied.
In practice, it has efficient approximation algorithms with provable guarantees, such as $k$-means++~\cite{arthur2006k} which yields $O(\log{k})$ approximation, using $D^2$ sampling.

The mean ($k=1$) $\sum_{p\in P}p/n$ of a set $P$ of $n$ points in $\REAL^d$ can be computed in linear $O(nd)$ time. However, we could not find in the literature an algorithm for computing even the sets-mean in~\eqref{1mean} for $n$ pairs of points on the plane ($m= d=2$).

\begin{figure}[t!]
\centering
    \subfigure[]{
		\includegraphics[width = 0.2\textwidth]{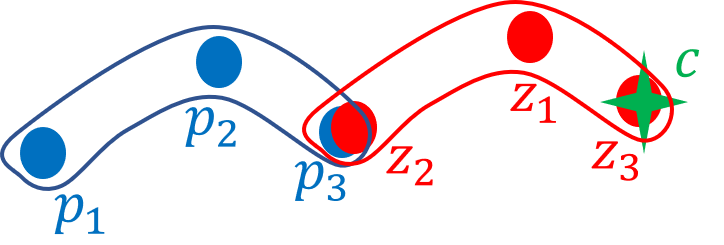}
		\label{fig:nonMetric}}
    \rulesep
    \subfigure[]{
    \centering
		\includegraphics[width = 0.15\textwidth]{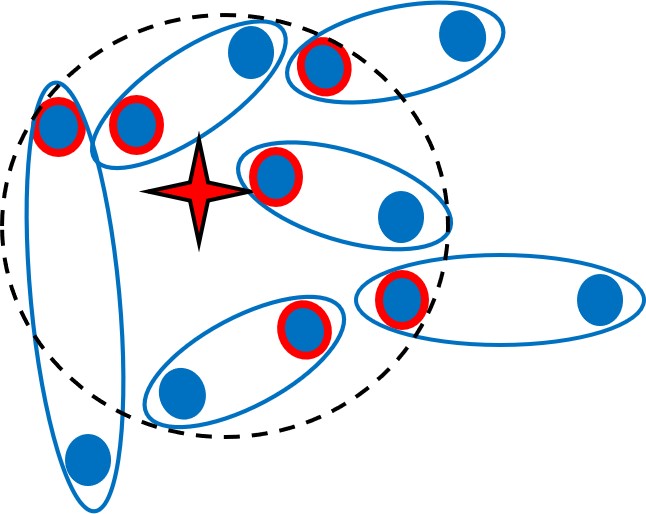}
        \label{fig:kmeansSetsBalls}}
    \caption{\textbf{Why is the sets clustering problem hard?} \ref{fig:nonMetric}: The space is non-metric. Two $m$-sets $P = \br{p_1,p_2,p_3}$ and $Z = \br{z_1,z_2,z_3}$ in $\REAL^d$ for $m=3$ and $c \in \REAL^d$ that do not satisfy the triangle inequality since $\Dt(P,Z) = 0, \Dt(Z,c) = 0$ but $\Dt(P,c) \neq 0$. \ref{fig:kmeansSetsBalls}: Separability. For $d=2$, a set of $n=6$ pairs (blue ellipses) and their optimal mean (red star). There is no ball that separates the closest $n$ points ($1$ from each set) which are closest to the optimal mean (red circles), from the other $n$ points (solid blue circles).}
    \label{fig:setsClusteringIsHard}
\end{figure}

\textbf{Separability. }
The clusters in the $k$-means problem are \emph{separable}: the minimum enclosing ball of each cluster consists only of the points in this cluster. Fundamental results in computational geometry~\cite{toth2017handbook} (chapter $28$) or PAC-learning theory~\cite{shalev2014understanding} prove that there are only $n^{O(1)}$ partitions of $P$ into $k$ such clusters that can be covered by balls.
On the contrary, even in the case of sets-mean $(k=1)$, the union of $n$ representative points from each pair is not separable from the other $n$ points (that are not served by the center); see Fig~\ref{fig:setsClusteringIsHard}.

\textbf{Non-metric space. }
The generalization of the $k$-means distance function to sets in~\eqref{kmeaneq123} is not a metric space, i.e., does not satisfy the triangle inequality, even approximately. For example, two input sets might have zero distance between them while one is very far and the other is very close to a center point; see Fig~\ref{fig:setsClusteringIsHard}.

\subsection{How Hard?}
The previous section may raise the suspicion that sets-$k$-means is NP-hard, even for $k=1$ and $d=2$.
However, this is not the case. In Section~\ref{Polynomialptassec}, we present a simple theorem for computing the \emph{exact} (optimal) sets-$k$-means for any input set $\set{P}$ of $n$ sets, each of size $m$. This takes time polynomial in $n$, i.e., $n^{O(1)}$, for every constant integers $k,d,m\geq1$.
The theorem is based on a generic reduction for the case of $k=m=1$.
Unfortunately, the constants that are hidden in the $O(1)$ notation above make our algorithm impractical for even modest values of $k$. This motivates the construction of the first \emph{coreset for sets}, which is the main technical result of this paper.

\subsection{Sets Coresets} \label{sec:coresets}
Coreset (or core-set) is a modern data summarization paradigm~\cite{maalouf2019fast,bachem2017practical,phillips2016coresets} that was originated from computational geometry~\cite{agarwal2005geometric}. Usually, the input for a \emph{coreset construction algorithm} is an approximation error $\eps\in(0,1)$, a set $\set{P}$ of $n$ items (called points), and a loss $\sum_{P\in \set{P}}\Dt(P,\cdot)$ that we wish to minimize over a (usually infinite) set $\C$ of feasible queries (solutions). The output is a (sub)set $\set{S} \subseteq \set{P}$ and a weights function $v:\set{S} \to [0,\infty)$, which is called an \emph{$\varepsilon$-coreset} for the tuple $(\set{P},\C,\Dt)$
if
\[
\abs{\sum_{P\in \set{P}}\Dt(P,C) - \sum_{S\in \set{S}}v(S)\Dt(S,C)} \leq \varepsilon \sum_{P\in \set{P}}\Dt(P,C),
\]
for \emph{every} query $C\in \C$. In particular, an optimal solution
of the coreset is an approximated optimal solution to the original problem.
If $\abs{\set{S}} \ll \abs{\set{P}}$, i.e., the size of the coreset $\set{S}$ is smaller than $\set{P}$ by orders of magnitude, then we can run a possibly inefficient algorithm on $\set{S}$ to compute an approximation solution to $\set{P}$.
In this paper, unlike previous papers, $\set{P}$ is a set of sets of size $m$ (rather than points) in $\REAL^d$ and $\mathcal{C} = \br{C\subseteq \REAL^d\mid |C|=k}$.

\paragraph{Why coresets? }
Applying the above optimal exhaustive search on such a coreset would reduce the running time from $n^{O(1)}$ to time near linear in $n$ conditioned upon: (i) every such input $\set{P}$ has a coreset $\set{S}$ of size, say, $\abs{\set{S}}\in (\log n)^{O(1)}$, and (ii) this coreset can be computed in near linear time, say $O(n\log n)$.

However, such a coreset construction for a problem has many other applications, including handling big streaming dynamic distributed data in parallel. Here, streaming means maintaining the sets-$k$-means of a (possibly infinite) stream of sets, via one pass and using only logarithmic memory and update time per new set. Dynamic data supports also deletion of sets. Distributed data means that the input is partitioned among $M\geq 2$ machines, where the running time reduces by a factor of $M$~\cite{regin2013embarrassingly}.
Many surveys explain how to obtain those applications, given an efficient construction of a small coreset as suggested in our paper. Due to lack of space we do not repeat them here and refer the reader to e.g.~\cite{feldman2020core}.

The recent result above~\cite{marom2019k} for $k$-means of lines (infinite sets) is obtained via coresets. We do not know any coresets for finite sets except for singletons ($m=1$). This coreset, that is called coreset for $k$-means (of points) is one of the fundamental and most researched coresets in this century:~\cite{har2004coresets, chen2006k,frahling2008fast, chen2009coresets,fichtenberger2013bico, bachem2015coresets, barger2016k, bachem2017one, feldman2017coresets,bachem2018scalable,huang2018epsilon}.  Coresets for fair clustering of points, which preserve sets-related properties of the input points, were suggested in~\cite{schmidt2019fair}.

A natural open question is \textbf{``does a small coreset exist for the sets-$k$-means problem of any input?"}.

\subsection{Main Contributions}
In this paper we suggest the first $(1+\eps)$ approximation for the sets-$k$-means problem, by suggesting the first coreset for sets. More precisely, we provide

\textbf{(i): }A proof that an $\eps$-coreset $\set{S}$ of size $\abs{\set{S}}=O(\log^2{n})$ exists for \emph{every} input set $\set{P}$ of $n$ sets in $\REAL^d$, each of size $m$. This holds for every constants $d,k,m\geq 1$. $\set{S}$ can be computed in time $O(n\log n)$; see exact details in Theorem~\ref{theorem:coreset}.

\textbf{(ii): }An algorithm that computes an optimal solution for the sets-$k$-means of such $\set{P}$ in $n^{O(1)}$ time. See Theorem~\ref{theorem:PTAS}.

\textbf{(iii): }Combining the above results implies the first PTAS (($1+\eps)$-approximation) for the sets-$k$-means of any such input set $\set{P}$, that takes $O(n\log n)$ time; see Corollary~\ref{cor:PTASKmeans}.

\textbf{(iv): }Extensions for (i) from the Euclidean distance in $\REAL^d$ to any metric space $(\M,\Dt)$, distances to the power of $\ell>0$, and M-estimators that are robust to outliers. See Section~\ref{sec:extension}.

\textbf{(v): }Experimental results on synthetic and real-world datasets show that our coreset performs well also in practice.

\textbf{(vi): }Open source implementation for reproducing our experiments and for future research~\cite{opencode}.

\subsection{Novelty} \label{sec:novelty}
Our coreset construction needs to characterize which of the input items are similar, and which are dissimilar, in some sense. To this end, we first suggest a similarity measure for sets and then present our novel non-uniform sampling scheme for sets, which we call \emph{onion sampling}.

\textbf{Recursive similarity. }When $m=1$, items are similar if their mutual distance is small. When $m \geq 2$, we propose a recursive and abstract similarity measure, which requires all the $m$ items in the first set to be ``close'' to the $m$ items in the second set, for some ordering of the items inside each set; see Algorithm~\ref{alg:RobustMedSets}.

\textbf{Onion Sampling. }
Recall that the $D^2$ sampling assigns each input point with probability that is proportional to its distance to the $k$-means of the input (or its approximation), which reflects its importance. When we try to generalize $D^2$ to handle sets rather than points, it is not clear what to do when one point in an input $m$-set is close to the approximated center and the other one is far, as in Fig.~\ref{fig:nonMetric}. In particular, if the optimal sum of squared distances is zero, the coreset in the in $k$-means problem is trivial (the $k$ points). This is not the case for the sets-$k$-mean (even for $k=1$).

To this end, we suggest an iterative and non-trivial alternative sampling scheme called onion sampling. In each iteration we apply an algorithm which characterizes ``recursively similar'' input sets, as described above, which form an ``onion layer''. We assign those sets the same sampling probability, which is inversely proportional to the number of those items, and peal this layer off.
We continue until we have pealed off the entire onion (input).
Finally, we prove that a random sample according to this distribution yields a coreset for the sets clustering problem; see Algorithm~\ref{alg:wrapper}.

\section{Definitions}
\label{sec:problemStatement} \label{sec:extension}

In~\eqref{kmeaneq123} we define sets-$k$-means for points in $\REAL^d$. However, our coreset construction holds for any metric space, or general (non-distance) loss functions as in Table~\ref{table:distFuncs}.
\begin{definition} [Loss function $\Dt$]
\label{def:distSets}
Let $\D:[0,\infty) \to [0,\infty)$ be a non-decreasing function that satisfies the following \emph{$r$-log-log Lipschitz condition}: There is a constant $0 < r < \infty$ such that for every $x,z > 0$ we have $\D(zx) \leq z^r \D(x)$.
Let $(\M,D)$ be a metric space, and $\Dt : \power{\M} \times \power{\M} \to [0,\infty)$ be a function that maps every two subsets $P,C \subseteq \M$ to
\[
\Dt(P,C) := \min_{p \in P, c \in C} \D(\Dt(p,c)).
\]
For $p,b\in  \M$, denote $\Dt(p,C):=\Dt(\br{p},C)$, and $\Dt(P,b):=\Dt(P,\br{b})$, for short. For an integer $k\geq 1$ define $\Set{X}_{k} := \br{ C \subseteq \M \mid  |C| = k}$.
\end{definition}

Although $(\M,\Dt)$ is not necessarily a metric space, the triangle inequality is approximated as follows.
\begin{lemma} [Lemma 2.1 (ii) in~\cite{feldman2012data}] \label{lem:weakTriangle}
Let $(\M,\dist)$ and $r>0$ be as defined in Definition~\ref{def:distSets}. Let $\rho = \max\br{2^{r-1},1}$. Then the function $\Dt$ satisfies the weak triangle inequality for singletons, i.e., for every $p,q,c \in \M$, $\Dt(p,q) \leq \rho(\Dt(p,c) + \Dt(c,q))$.
\end{lemma}

\begin{table}[th!]
\caption{Example loss functions as in Definition~\ref{def:distSets}. Let $\delta > 0$ be a constant and let $(\M,\Dt)$ be a metric space where $\M = \REAL^d$ and $\Dt(p,c) = \norm{p-c}$ for every $p,c \in \REAL^d$.}
\centering
\begin{adjustbox}{width=0.48\textwidth}
\small
\begin{tabular}{ | c | c | c | c |}
\hline
\makecell{Optimization\\Problem} & $\D(x)$ & $\Dt(P,C)$ & $\rho$\\
\hline
\makecell{sets-$k$-median}  & $x$ & $\displaystyle\min_{p\in P, c\in C}\norm{p-c}$ & $1$\\ \hline
\makecell{sets-$k$-means}  & $x^2$ & $\displaystyle\min_{p\in P, c\in C}\norm{p-c}^2$ & $2$\\
\hline
\makecell{sets-$k$-means with\\$M$-estimators} &
$\begin{cases}
\frac{1}{2}x^2 & \text{if } x\leq \delta\\
\delta(|x| - \frac{1}{2}\delta) & \text{otherwise}
\end{cases}$ &
$\displaystyle\min_{p\in P, c \in C}\begin{cases}
\frac{1}{2}\norm{p-c}^2 & \text{if } \norm{p-c}\leq \delta\\
\delta(\norm{p-c} - \frac{1}{2}\delta) & \text{otherwise}
\end{cases}$
&
$2$\\
\hline
\makecell{$\ell_\psi$ norm}  &  $x$ & $\displaystyle\min_{p\in P, c\in C}\norm{p-c}_\psi$ & $\max\br{2^{\frac{1}{\psi}}, 1}$\\
\hline
\end{tabular}
\end{adjustbox}
\label{table:distFuncs}
\end{table}

\textbf{Notation. }For the rest of the paper we denote $[n] = \br{1,\cdots,n}$ for an integer $n\geq 1$. Unless otherwise stated, let $(\M,\Dt)$ be as in Definition~\ref{def:distSets}.

As discussed in Section~\ref{sec:intro}, the input set for the sets clustering problem is a set of finite and equal sized sets as follows.
\begin{definition} [$(n,m)$-set]
\label{def:mSet}
An $m$-set $P$ is a set of $m$ \emph{distinct} points in $\M$, i.e. $P \subseteq \M$ and $|P| = m$.
An \emph{$(n,m)$-set} is a set $\Set{P} = \br{P \mid P \subseteq \M, \abs{P} = m}$ such that $\abs{\Set{P}} = n$.
\end{definition}

In what follows we define the notion of robust approximation. Informally, a robust median for an optimization problem at hand is an element $b$ that approximates the optimal value of this optimization problem, with some leeway on the number of input elements considered.
\begin{definition} [Robust approximation]
\label{def:opt} \label{def:robMed}
Let $\Set{P}$ be an $(n,m)$-set, $\gamma \in (0,\frac{1}{2}]$, $\tau \in (0, 1/10)$, and $\alpha \geq 1$. Let $(\M,\Dt)$ be as in Definition~\ref{def:distSets}.
For every $C \in \set{X}_k$, we define $\closest(\Set{P}, C, \gamma)$ to be the set that is the union of $\ceil{\gamma |\Set{P}|}$ sets $P \in \Set{P}$ with the smallest values of $\Dt(P,C)$, i.e.,
\[
\closest(\set{P}, C,\gamma) \in \argmin\limits_{\substack{\set{Q} \subseteq \set{P} : \abs{Q} = \ceil{\gamma \abs{\set{P}}}}} \sum_{P\in \set{Q}}\Dt(P,C).
\]
The singleton $\br{b} \in \set{X}_1$ is a $(\gamma,\tau, \alpha)$-\emph{median} for $\set{P}$ if
\begin{equation*}
\smashoperator[r]{\sum\limits_{P \in \closest(\Set{P},\br{b},(1 - \tau) \gamma)}} \Dt(P,b) \leq \alpha \cdot \min_{b'\in \M} \smashoperator[r]{\sum_{P\in \closest(\Set{P},\br{b'},\gamma)}} \Dt(P,b').
\end{equation*}

\end{definition}

Given an $m$-set $P$, and a set $\set{B}$ of $|\set{B}| = j \leq m$ points, in what follows we define the projection of $P$ onto $\set{B}$ to be the set $P$ after replacing $j$ of its points, which are the closest to the points of $\set{B}$, by the points of $\set{B}$. We denote by $\notail(P,\set{B})$ the remaining ``non-projected'' points of $P$.
\begin{definition} [Set projection]
\label{def:proj}
Let $m\geq 1$ be an integers, ${P}$ be an $m$-set, $(\M,\Dt)$ be as in Definition~\ref{def:distSets}, $j\in [m]$, and let $\set{B} = \br{b_1, \cdots ,b_j} \in \set{X}_j$. Let $p_1 \in P$ denote the closest point to $b_1$ i.e., $p_1\in \argmin_{p\in P}\Dt(p, b_1)$. For every integer $\displaystyle{i\in \br{2,\cdots,j}}$ recursively define $p_i\in P$ to be the closest point to $b_i$, excluding the $i-1$ points that were already chosen, i.e., $p_i\in \argmin\limits_{p\in P\setminus\br{p_1, \cdots,p_{i-1}}}\Dt(p, b_i)$ .
We denote \textbf{(i): }$\br{(p_1,b_1),\cdots,(p_j,b_j)}$ by $\closepoints(P,\set{B})$,\\
\textbf{(ii): }the $m-j$ points from $P$ that are not among the closest points to $\set{B}$ by $\notail(P,\set{B})= P\setminus \br{p_1,\cdots,p_j}$, and\\
\textbf{(iii): }the \emph{projection} of $P$ onto $\set{B}$ by $\proj(P,\set{B}) = \br{b_1,\cdots,b_j} \cup \big(  P\setminus \br{p_1,\cdots,p_j} \big)$.
For $X = \emptyset$, we define $\notail(P,X) = \proj(P,X) = P$

\end{definition}

\section{Sensitivity Based Coreset}\label{sec:SensitivitySampling}

A common technique to compute coresets is the approach of non-uniform sampling, which is also called sensitivity sampling~\cite{langberg2010universal,braverman2016new}, and was widely used lately to construct coresets for Machine Learning problems; see e.g.,~\cite{huggins2016coresets,munteanu2018coresets,maalouf2019tight,bachem2017practical}.
Intuitively, the sensitivity of an element $P \in \set{P}$ represents the importance of $P$ with respect to the other elements, and the specific optimization problem at hand; see definition and details in Theorem~\ref{braverman}. Suppose that we computed an upper bound $s(P)$ for the sensitivity of every element $P\in\set{P}$. Then a coreset is now simply a random (sub)sample of $\set{P}$ according to the sensitivity distribution, followed by a smart reweighting of the points.
It's size is proportional to the sum of sensitivities $t = \sum_{Q\in\set{P}}s(Q)$ and the combinatorial complexity $d'$ of the problem at hand; see Definition~\ref{def:VC}. The following theorem, which is a restatement of Theorem 5.5 in~\cite{braverman2016new}, provides full details.

\begin{theorem} \label{braverman}
Let $\set{P}$ be an $(n,m)$-set, and $(\Dt,\set{X}_k)$ be as in Definition~\ref{def:distSets}. For every $P\in \set{P}$ define the \emph{sensitivity} of $P$ as
\[
\sup_{C\in \set{X}_k} \frac{\Dt(P,C)}{\sum_{Q\in \Pset}\Dt(Q,C)},
\]
where the sup is over every $C\in \set{X}_k$ such that the denominator is non-zero.
Let $s:\Pset\to [0,1]$ be a function such that $s(P)$ is an upper bound on the sensitivity of $P$.
Let $t = \sum_{P\in \Pset} s(P)$ and $d'$ be a complexity measure of the set clustering problem; see Definition~\ref{def:VC}. Let $c \geq 1$ be a sufficiently large constant, $\varepsilon, \delta \in (0,1)$, and let $\set{S}$ be a random sample of $|S| \geq \frac{ct}{\varepsilon^2}\left(d'\log{t}+\log{\frac{1}{\delta}}\right)$
sets from $\Pset$, such that $P$ is sampled with probability $s(P)/t$ for every $P\in \Pset$. Let $v(P) = \frac{t}{s(P)|C|}$ for every $P\in \set{S}$. Then, with probability at least $1-\delta$, $(S,v)$ is an $\varepsilon$-coreset for $(\Pset,\set{X}_k,\Dt)$.
\end{theorem}

\section{Coreset for Sets Clustering}
In this section we give our main algorithms that compute a coreset for the sets clustering problem, along with intuition, Full theoretical proofs can be found in the appendix.

\subsection{Algorithms} \label{sec:algs}

\textbf{Overview and intuition behind Algorithm~\ref{alg:RobustMedSets}. }
Given a set $\set{P}$ of $m$-sets and an integer $k \geq 1$, Algorithm~\ref{alg:RobustMedSets} aims to compute a set $\set{P}^m \subset \set{P}$ of ``similar'' $m$-sets, which are all equally important for the problem at hand; see Lemma~\ref{lem:sens}.
At the $i$th iteration we wish to find a $\frac{1}{4k}$ fraction of the remaining $m$-sets $\set{P}^{i-1}$ which are similar in the sense that there is a dense ball of small radius that contains at least one point from each of those sets. To do so, we first compute at Line~\ref{alg:LinePHati} ${\hat{\set{P}}}^{i-1}$ which contains only the ``non-projected'' points of each $m$-sets in $\set{P}^{i-1}$. We then compute a median $b^i$ at Line~\ref{alg:LineBi} that satisfies at least $\frac{1}{4k}$ of ${\hat{\set{P}}}^{i-1}$. $b^i$ is the center of the desired dense ball. At Line~\ref{newPi} we pick the sets that indeed have a candidate inside this dense ball and continue to the next iteration (where again, we consider only the non-projected part of those sets); see Fig.~\ref{fig:illustration}.
After $m$ such iterations, the surviving $m$-sets in $\set{P}^m$ have been ``recursively similar'' throughout all the iterations.

\textbf{Overview and intuition behind Algorithm~\ref{alg:wrapper}. }
Given an $(n,m)$-set $\set{P}$ and an integer $k\geq 1$, Algorithm~\ref{alg:wrapper} aims to compute an $\varepsilon$-coreset $(S,v)$ for $\set{P}$; see Theorem~\ref{theorem:coreset}. Algorithm~\ref{alg:wrapper} applies our onion sampling scheme; each while iteration at Line~\ref{line:while} corresponds to a pealing iteration.

At lines~\ref{line:while}--\ref{line:endSens} Algorithm~\ref{alg:wrapper} first calls Algorithm~\ref{alg:RobustMedSets} with the $(n,m)$-set $\set{P}^0 = \set{P}$ to obtain a set $\set{P}^m \subseteq \set{P}$ of ``dense'' and equally (un)important $m$-sets from the input. Second, it assigns all the sets in $\set{P}^m$ the same sensitivity value as shown in Lemma~\ref{lem:sens}. It then peals those sets off, and repeats this process with $\set{P}^0 \setminus \set{P}^m$. Those values increase in every step since the size of the dense set returned decreases, making every point more important. This process is illustrated in Fig.~\ref{fig:sens}. We then randomly sample a sufficiently large set $S \subseteq \set{P}$ at Line~\ref{line:randomSample} according to the sensitivity values, and assign new weights $v(P)$ for every set $P \in \set{S}$ in Line~\ref{line:reweight}.

\begin{figure}
  \centering
  \includegraphics[width=0.47\textwidth]{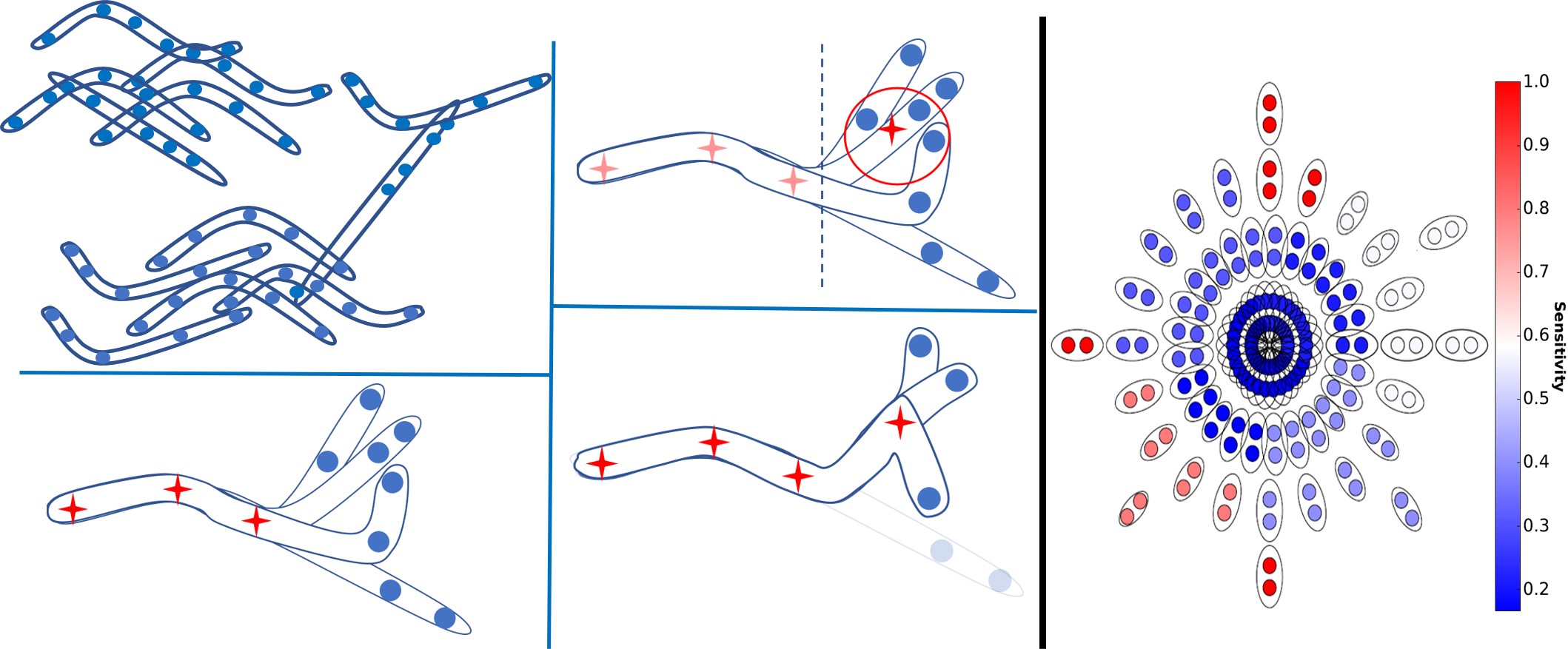}
  \caption{\textbf{Recursive similarity. }(Upper left): An input $(n,m)$-set $\set{P}$ with $n = 14$ and $m = 5$. (Lower left): The set $\set{B}^i$ (red stars) for $i=3$ from the $3rd$ iteration of Algorithm~\ref{alg:RobustMedSets}, and the projection $\proj(\set{P}^3,\set{B}^3)$ (blue snakes) of $\set{P}^3$ onto $\set{B}^3$. Therefore, all the sets in $\proj(\set{P}^3,\set{B}^3)$ have $i=3$ points in common, and $m-i = 2$ other points. (Upper mid): The set $\hat{\set{P}}^3$ that contains four $2$-sets (blue points right of the dashed line). The median $b^4$ (bold red star) considers only a fraction of $\hat{\set{P}}^3$. The red circle contains the points that are closest to $b^4$.
    $\set{P}^4$ contains the sets with a representative inside the red circle. (Lower mid): The projection $\proj(\set{P}^4,\set{B}^4)$ (blue snakes) of the sets $\set{P}^4$ onto the new $\set{B}^4 = \set{B}^3 \cup \br{b^4}$.\\
    \textbf{(Right): Onion sampling. }A set $\set{P}$ of pairs in the plane ($m=d=2$) along with the sensitivity values $s(P)$ computed in Algorithm~\ref{alg:wrapper} via our onion sampling. First, the densest subset of pairs are assigned a low sensitivity value (dark blue). The densest subset of the remaining pairs is then assigned a higher sensitivity value (light blue), and so on. The scattered pairs that remain at the end are assigned the highest sensitivity (dark red). The size of the subset found decreases in each step.}
  \label{fig:sens}\label{fig:illustration}
\end{figure}

\begin{algorithm}[tb]
  \caption{$\RobustMedForSets(\Pset,k)$}
   \label{alg:RobustMedSets}
\begin{algorithmic}[1]
   \STATE {\bfseries Input:} An $(n,m)$-set $\set{P}$ and  a positive integer $k$.
    \STATE {\bfseries Output:} A pair $(\set{P}^m,\set{B}^m)$ where $\set{P}^m \subseteq \set{P}$ and \\
    \quad\quad\quad\,\,\,\,$\set{B}^m \in \set{X}_m$; see Lemma~\ref{lem:sens}.
    \STATE $\set{P}^0 := \set{P}$ and $\set{B}^0 := \emptyset$
    \FOR{$i := 1$ {\bfseries to} $m$  \alglinelabel{alg:forLine}}
    \STATE ${\hat{\set{P}}}^{i-1}:= \br{ \notail(P,\set{B}^{i-1}) \mid  P \in \set{P}^{i-1}}$ \COMMENT{see Definition~\ref{def:proj}}\alglinelabel{alg:LinePHati}
    \STATE Compute a $\displaystyle \left(\frac{1}{2k},\tau, 2\right)$-median $\br{b^{i}} \in \set{X}_1$ for ${\hat{\set{P}}}^{i-1}$ for some $\tau \in (0,\frac{1}{20})$.\\ \COMMENT{see Definition~\ref{def:robMed}. Algorithm~\ref{alg:robustMed} provides a suggested implementation.} \alglinelabel{alg:LineBi}
    \STATE $\set{{P}}^{i} := \left\lbrace \mkern-17mu \begin{array}{c!{\vline width 0.6pt}c}
         \begin{array}{c}
              P \\
              \in \set{P}^{i-1}
         \end{array} &  \begin{array}{c}
         \notail(P,\set{B}^{i-1}) \in \\ \closest\Big({\hat{\set{P}}}^{i-1},\br{b^{i}},\frac{(1-\tau)}{4k}\Big)
         \end{array}
    \end{array} \mkern-17mu \right\rbrace$ \alglinelabel{newPi}
    \COMMENT{$\set{{P}}^{i}$ contains every $m$-set $P$ such that $\notail(P,\set{B}^{i-1})$ is in the closest fraction of $(1-\tau)/(4k)$ sets in $\hat{\set{P}}^{i-1}$ to the center $b^i$; see Fig.~\ref{fig:illustration}.}

    \STATE $\set{B}^i := \set{B}^{i-1} \bigcup \br{b^{i}}$ \alglinelabel{line:compNewBi}
    \ENDFOR
    \STATE \bfseries{Return}$(\set{P}^m,\set{B}^m)$

\end{algorithmic}
\end{algorithm}

\begin{algorithm}[!htb]
  \caption{$\coreset(\Pset,k, \eps, \delta)$}
   \label{alg:wrapper}
\begin{algorithmic}[1]
   \STATE {\bfseries Input:} An $(n,m)$-set $\set{P}$, a positive integer $k$, an error \\
   \quad\quad\quad parameter $\eps \in (0,1)$, and a probability of failure \\
   \quad\quad\quad $\delta \in (0,1)$.
    \STATE {\bfseries Output:} An $\eps$-coreset $\left( S, v\right)$; see Theorem~\ref{theorem:coreset}.
    \STATE $b := $ a constant determined by the proof of Theorem~\ref{theorem:coreset}
    \STATE $d' := md^2k^2$ \COMMENT{the dimension of $(\Pset,\set{X}_k,\Dt)$}
    \STATE $\set{P}^0 := \set{P}$ \alglinelabel{line:P0}
    \WHILE{$\abs{Q_0} > b$ \alglinelabel{line:while}}
    \STATE \alglinelabel{line:callAlg1}$\left( \set{P}^m, \set{B}^m \right) := \RobustMedForSets(\set{P}^0,k)$
    \FOR{every $P \in \set{P}^m$}
    \STATE $s(P) := \frac{b}{\abs{\set{P}^m}}$ \alglinelabel{line:sens}
    \ENDFOR
    \STATE $\set{P}^0 := \set{P}^0 \setminus \set{P}^m$ \alglinelabel{alg2:L7}
    \ENDWHILE
    \FOR{every $P \in \set{P}^0$}
    \STATE \alglinelabel{line:endSens} $s(P) = 1$
    \ENDFOR
    \STATE $t := \sum\limits_{P \in \set{P}} s(P)$ \alglinelabel{line:sumSens}
    \STATE \alglinelabel{line:randomSample} Pick $S$ of $\frac{bt}{\eps^2}\left( \log{(t)} d' + \log\left( \frac{1}{\delta} \right) \right)$ $m$-sets from $\set{P}$ by repeatedly, i.i.d, selecting $P \in \set{P}$ with probability $\frac{s(P)}{t}$
    \FOR{each $P \in S$}
    \STATE $v(P) := \frac{t}{\abs{S} \cdot s(P)}$ \alglinelabel{line:reweight}
    \ENDFOR
    \STATE \bfseries{Return} $(S, v)$ \alglinelabel{line:ret}
\end{algorithmic}
\end{algorithm}

\subsection{Main Theorems}

The following lemma lies at the heart of our work. It proves that Algorithm~\ref{alg:RobustMedSets} helps compute an upper bound for the sensitivity term of some of the input elements.
\begin{restatable}{lemma}{sens}
\label{lem:sens}
Let $\set{P}$ be an $(n,m)$-set, $k\geq 1$ be an integer and $(\M,\Dt)$ be as in Definition~\ref{def:distSets}. Let $(\set{P}^m,\set{B}^m)$ be the output of a call to $\RobustMedForSets(\Pset,k)$; see Algorithm~\ref{alg:RobustMedSets}.
Then, for every $P \in \set{P}^m$ we have that
\[
\sup_{C \in \set{X}_k}\frac{\Dt(P,C)}{\sum\limits_{Q \in \set{P}} \Dt(Q,C)} \in O(1) \cdot \left(  {\frac{1}{\abs{\set{P}^m}}} \right).
\]
\end{restatable}

The following theorem is our main technical contribution. It proves that Algorithm~\ref{alg:wrapper} indeed computes an $\varepsilon$-coreset.
\begin{restatable}{theorem}{epscoreset}
\label{theorem:coreset}
Let $\set{P}$ be an $(n,m)$-set, $k\geq 1$ be an integer, $(\M,\Dt)$ be as in Definition~\ref{def:distSets}, and $\varepsilon, \delta \in (0,1)$. Let $(\set{S},v)$ be the output of a call to $\coreset(\Pset,k, \eps, \delta)$.
Then
\begin{enumerate}[label=(\roman*)]
  \item $|\set{S}| \in O\left(\left(\frac{md\log{n}}{\varepsilon}\right)^2k^{m+4}\right)$.
  \item With probability at least $1-\delta$, $(\set{S},v)$ is an $\varepsilon$-coreset for $(\set{P},\set{X}_k,\Dt)$; see Section~\ref{sec:coresets}.
  \item $(S,v)$ can be computed in $O(n\log(n)(k)^m)$ time.
\end{enumerate}
\end{restatable}

\subsection{Polynomial Time Approximation Scheme.} \label{Polynomialptassec}

In the following theorem we present a reduction from an $\alpha$-approximation for the sets clustering problem in $\REAL^d$ with $m,k \geq 1$, to an $\alpha$-approximation for the simplest case where $m=k=1$, for any $\alpha \geq 1$. We give a suggested implementation in Algorithm~\ref{alg:polynomials}.

\begin{restatable}{theorem}{thmPTAS}
 \label{theorem:PTAS}
Let $\set{P}$ be an $(n,m)$-set in $\REAL^d$, $w:\set{P}\to[0,\infty)$ be a weights function, $k\geq 1$ be an integer, $\alpha \geq 1$ and $\delta \in [0,1)$. Let $\Dt$ be a loss function as in Definition~\ref{def:distSets} for $\M = \REAL^d$.
Let $\oneMean$ be an algorithm that solves the case where $k=m=1$, i.e., it takes as input a set $Q \subseteq \M$, a weights function $u:Q\to[0,\infty)$ and the failure probability $\delta$, and in time $T(n)$ outputs $\hat{c} \in \M$ that with probability at least $1-\delta$ satisfies
$\sum_{q \in Q} u(q)\cdot \Dt(q,\hat{c}) \leq \alpha \cdot \min_{c \in \M} \sum_{q \in Q} u(q)\cdot\Dt(q,c)$.
Then in $T(n) \cdot (nmk)^{O(dk)}$ time we can compute $\hat{C} \in \set{X}_k$ such that with probability at least $(1-k\cdot\delta)$ we have
$$\sum_{P\in \set{P}}w(P)\cdot\Dt(P,\hat{C}) \leq \alpha\cdot \min_{C \in \set{X}_k} \sum_{P\in \set{P}}w(P)\cdot\Dt(P,C).$$
\end{restatable}

The previous theorem implies a polynomial time (optimal) solution for the sets-$k$-means, since it is trivial to compute an optimal solution for the case of $m=k=1$.

\begin{restatable}[PTAS for sets-$k$-means]{corollary}{corPTASKmeans}
\label{cor:PTASKmeans}
Let $\set{P}$ be an $(n,m)$-set, $k\geq 1$ be an integer, and put $\eps \in \left(0, \frac{1}{2}\right]$ and $\delta\in (0,1)$. Let
$\OPT$ be the cost of the sets-$k$-means.
Then in $n\log(n)(k)^m  + \left( \frac{\log{n}}{\varepsilon}dmk^m\right)^{O(dk)}$ time we can compute $\hat{C} \in \set{X}_k$ such that with probability at least $1-k\cdot\delta$,
\[
\sum_{P\in \set{P}}\min_{p\in P, c\in \hat{C}}\norm{p-c}^2 \leq (1+4\varepsilon)\cdot \OPT.
\]
\end{restatable}

\section{Robust Median}
In this section, we provide an algorithm that computes a robust approximation; see Definition~\ref{def:robMed} and its preceding paragraph. An overview is provided in Section~\ref{sec:suppMed}.

\begin{algorithm}[!htb]
   \caption{$\RobustMedAlg(\Pset,k,\delta)$}
   \label{alg:robustMed}
\begin{algorithmic}[1]
   \STATE {\bfseries Input:}  An $(n,m)$-set $\set{P}$, a positive integer $k \geq 1$, and \\
   \quad\quad\quad the probability of failure $\delta \in (0,1)$

   \STATE {\bfseries Output:}  A point $q \in \M$ that satisfies Lemma~\ref{theorem:robustMed}

   \STATE $b := $ a universal constant that can be determined from the proof of Lemma~\ref{theorem:robustMed} \alglinelabel{line:hmm}

   \STATE Pick a random sample $\set{S}$ of $|S| = b\cdot k^2 \log\left(\frac{1}{\delta}\right)$ sets from $\set{P}$ \alglinelabel{line:sampleS}

   \STATE $q := $ a point that minimizes

   $\sum_{p\in \closest(\set{S},\br{\hat{q}},(1-\tau)\gamma)} \Dt(p,\hat{q})$ over $\hat{q} \in Q \in \set{S}$ \alglinelabel{line:compq}

   \STATE \bfseries{Return} $q$ \alglinelabel{line:retMed}
\end{algorithmic}
\end{algorithm}

\begin{restatable}[based on Lemma $9.6$ in~\cite{feldman2011unified}]{lemma}{robustMed}
\label{theorem:robustMed}
Let $\set{P}$ be an $(n,m)$-set, $k \geq 1$, $\delta\in (0,1)$ and $(\set{X},\Dt)$ be as in Definition~\ref{def:distSets}. Let $q \in \M$ be the output of a call to $\RobustMedAlg(\Pset,k, \delta)$; see Algorithm~\ref{alg:robustMed}. Then with probability at least $1-\delta$, $q$ is a $(1/(2k),1/6,2)$-median for $P$; see Definition~\ref{def:robMed}. Furthermore, $q$ can be computed in $O\left(tb^2k^4\log^2\left(\frac{1}{\delta}\right)\right)$ time, where $t$ is the time it takes to compute $\Dt(P,Q)$ for $P,Q \in \set{P}$.
\end{restatable}

\section{Experimental Results} \label{sec:applications} \label{sec:ER}
\newcommand\s{0.18}
\begin{figure*}[th!]
\centering
    \subfigure[dataset (i), $k=2$]{
		\centering
		\includegraphics[width = \s\textwidth]{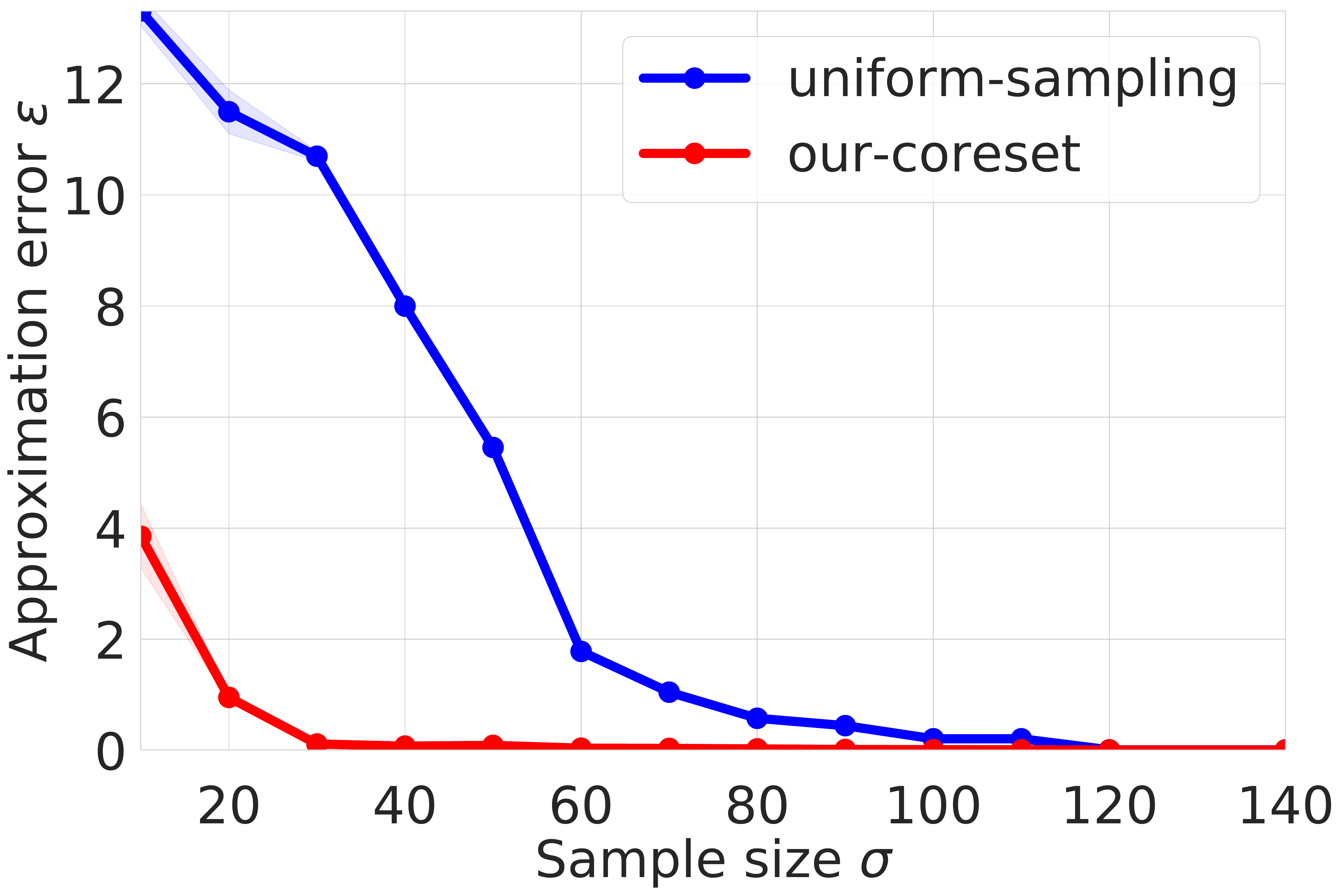}
        \label{fig:HW1}}
    \subfigure[dataset (i), $k=4$]{
		\centering
		\includegraphics[width = \s\textwidth]{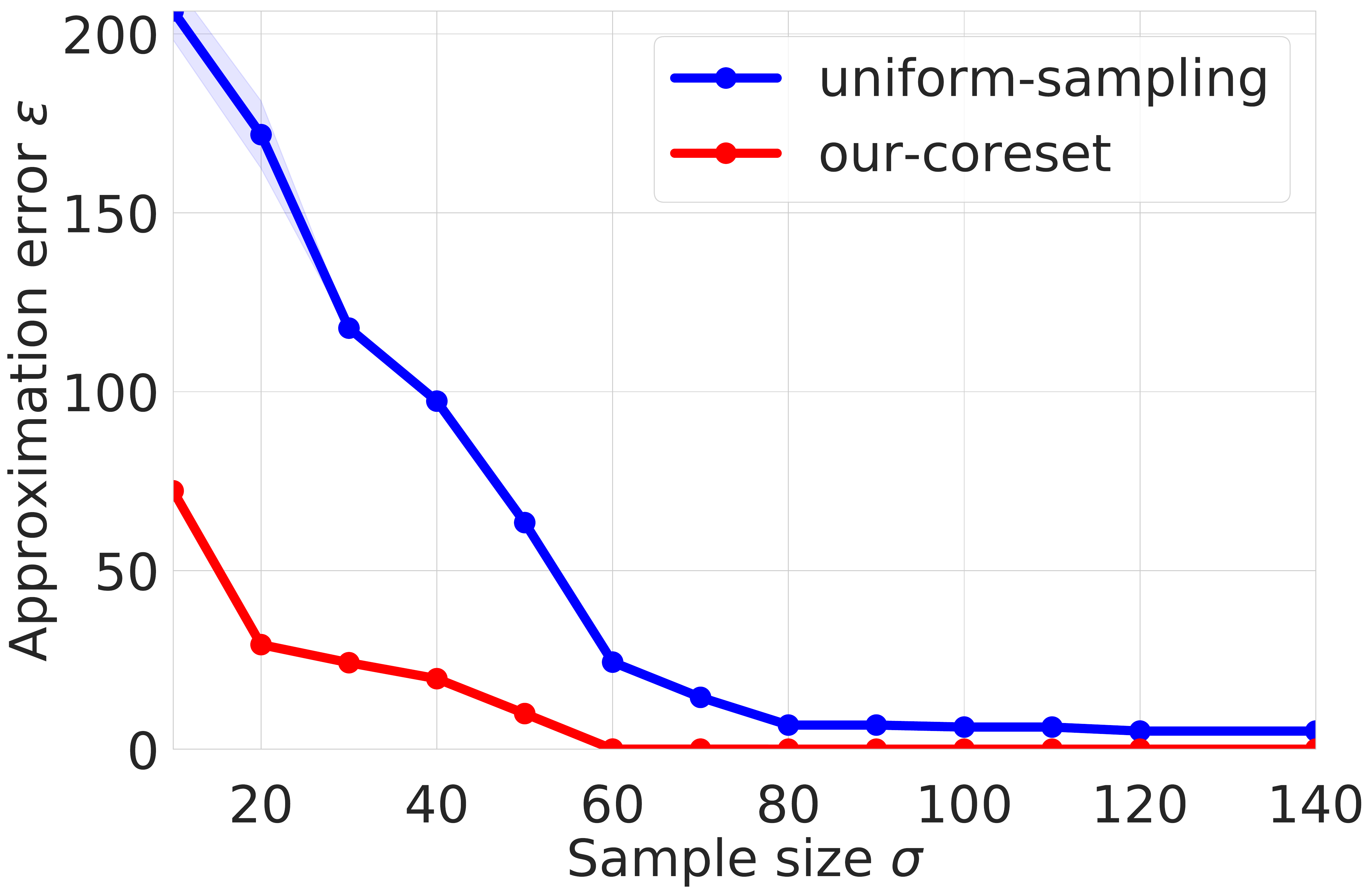}
        \label{fig:HW2}}
	\subfigure[dataset (i), $k=6$]{
		\centering
		\includegraphics[width = \s\textwidth]{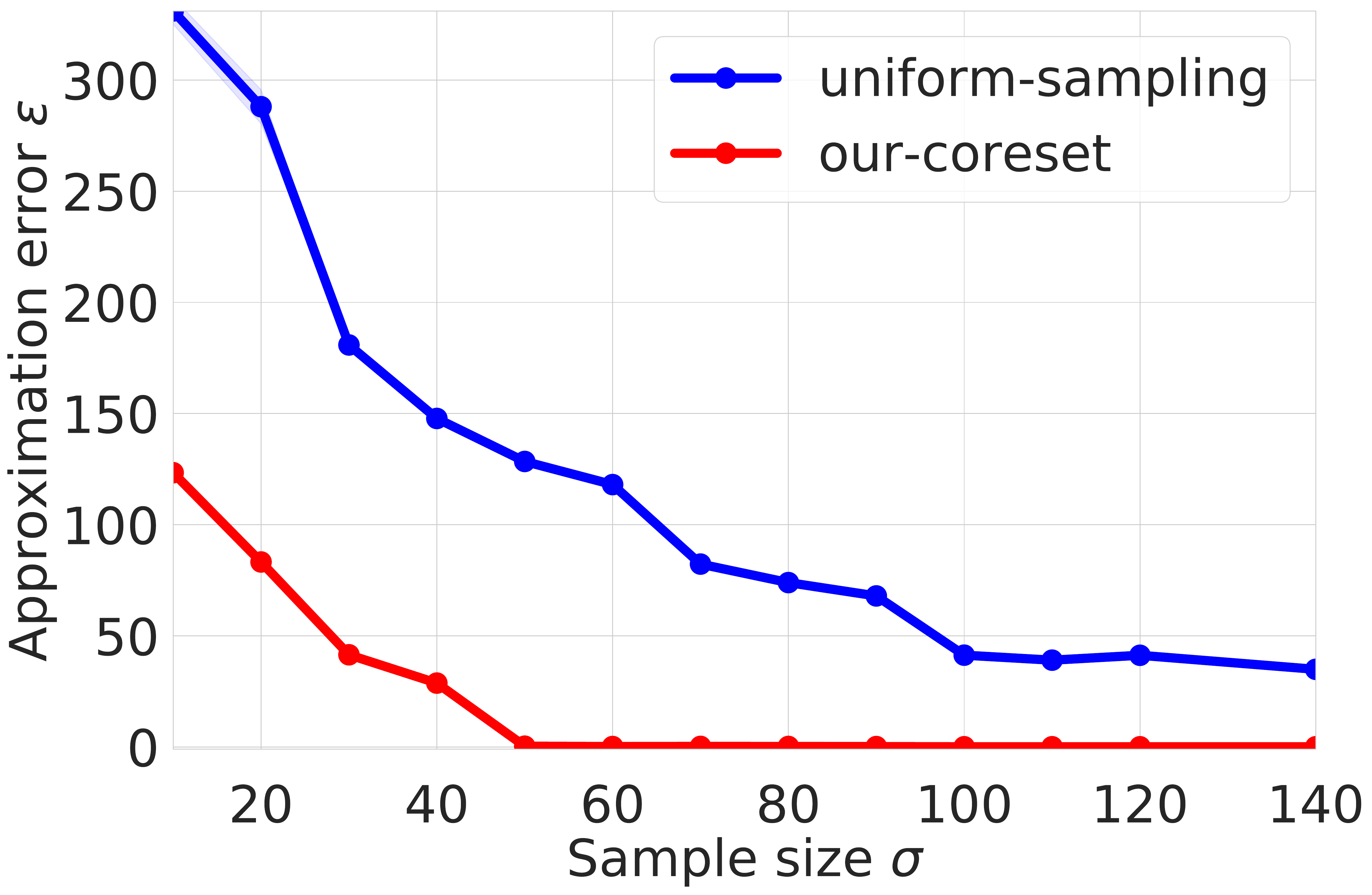}
        \label{fig:HW3}}
    \subfigure[dataset (i), $k=8$]{
		\centering
		\includegraphics[width = \s\textwidth]{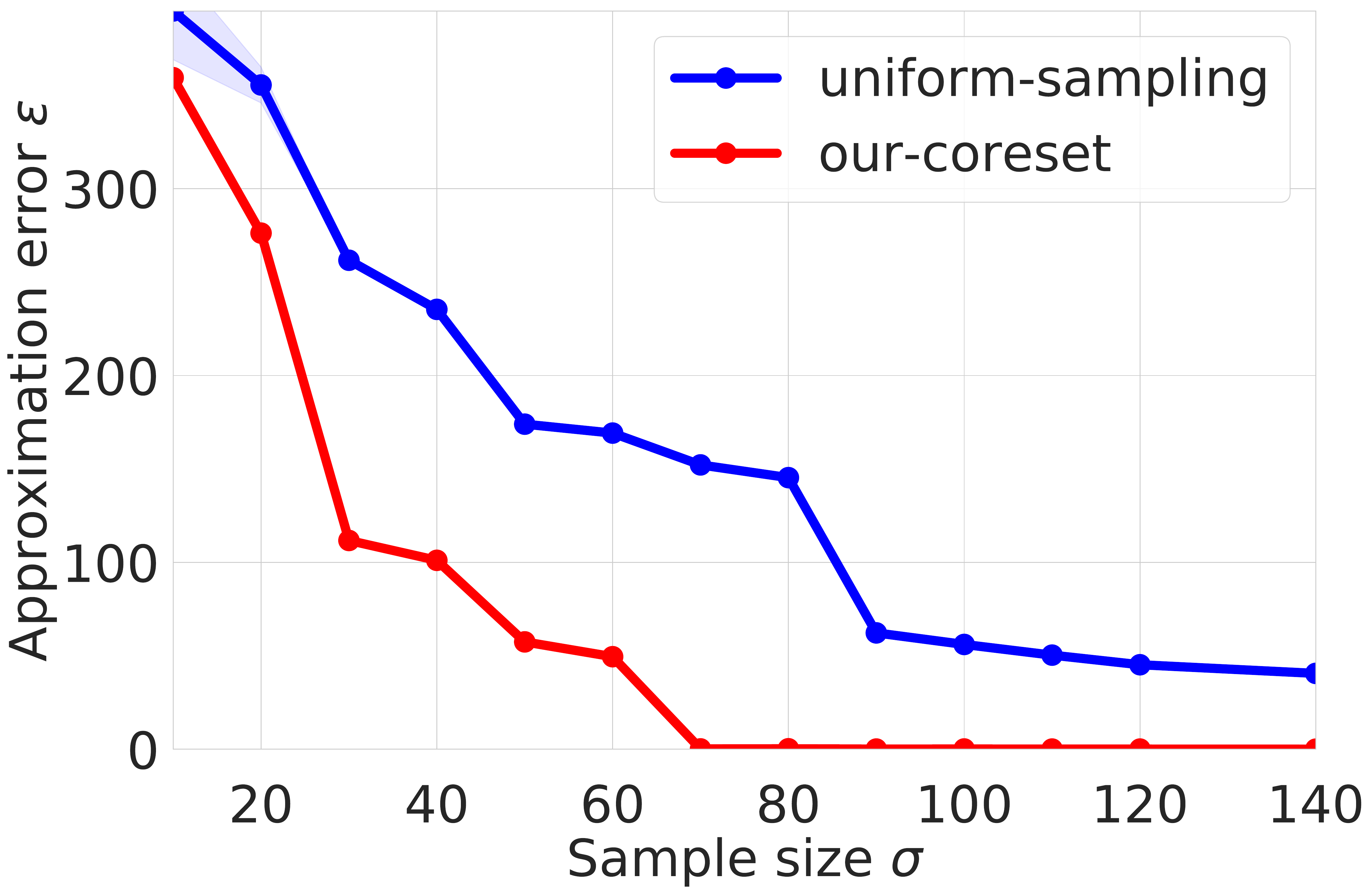}
        \label{fig:HW4}}
    \subfigure[dataset (ii), $k=4$]{
		\centering
		\includegraphics[width = \s\textwidth]{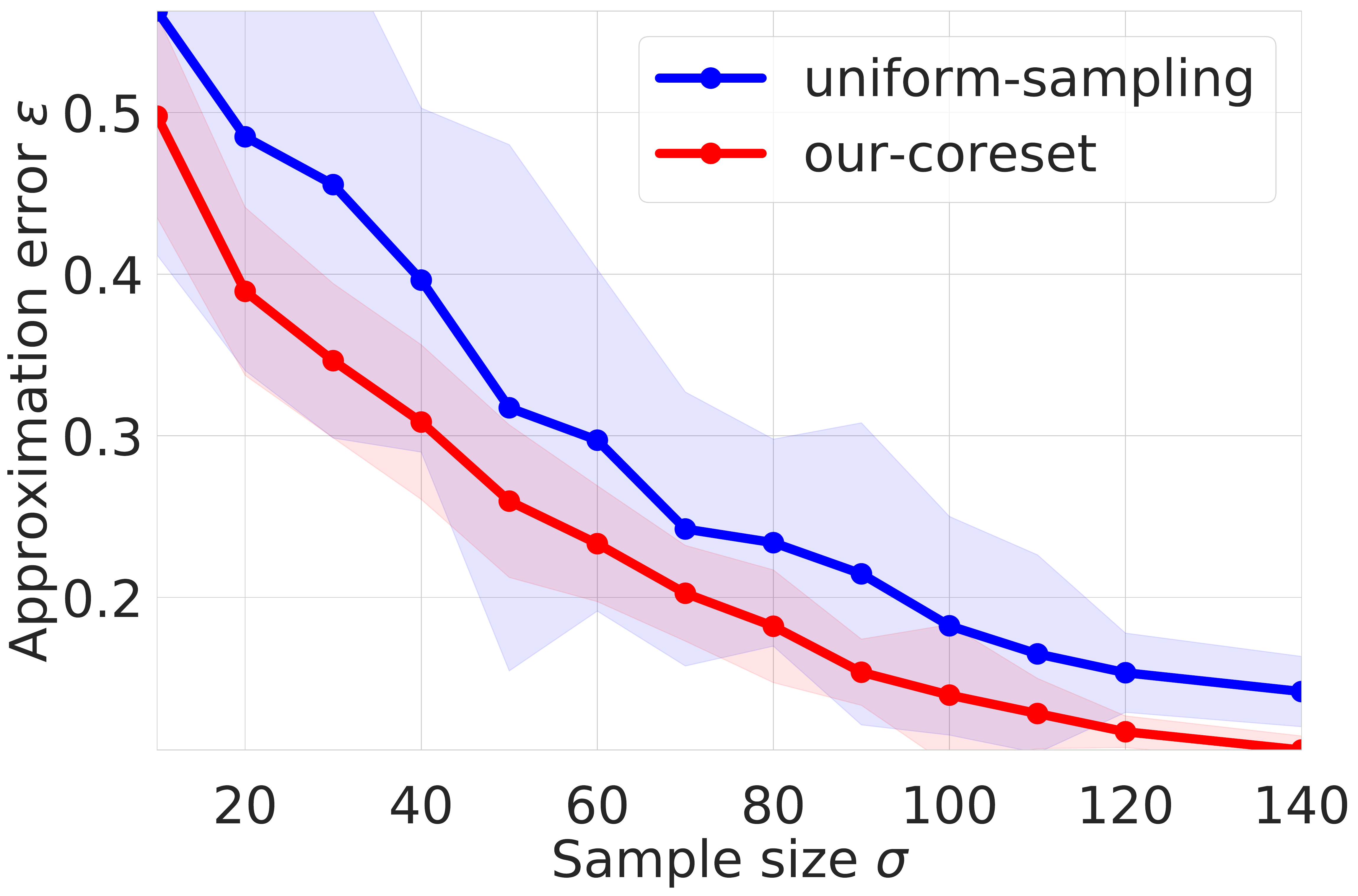}
        \label{fig:dicos1}}
    \subfigure[dataset (ii), $k=6$]{
		\centering
		\includegraphics[width = \s\textwidth]{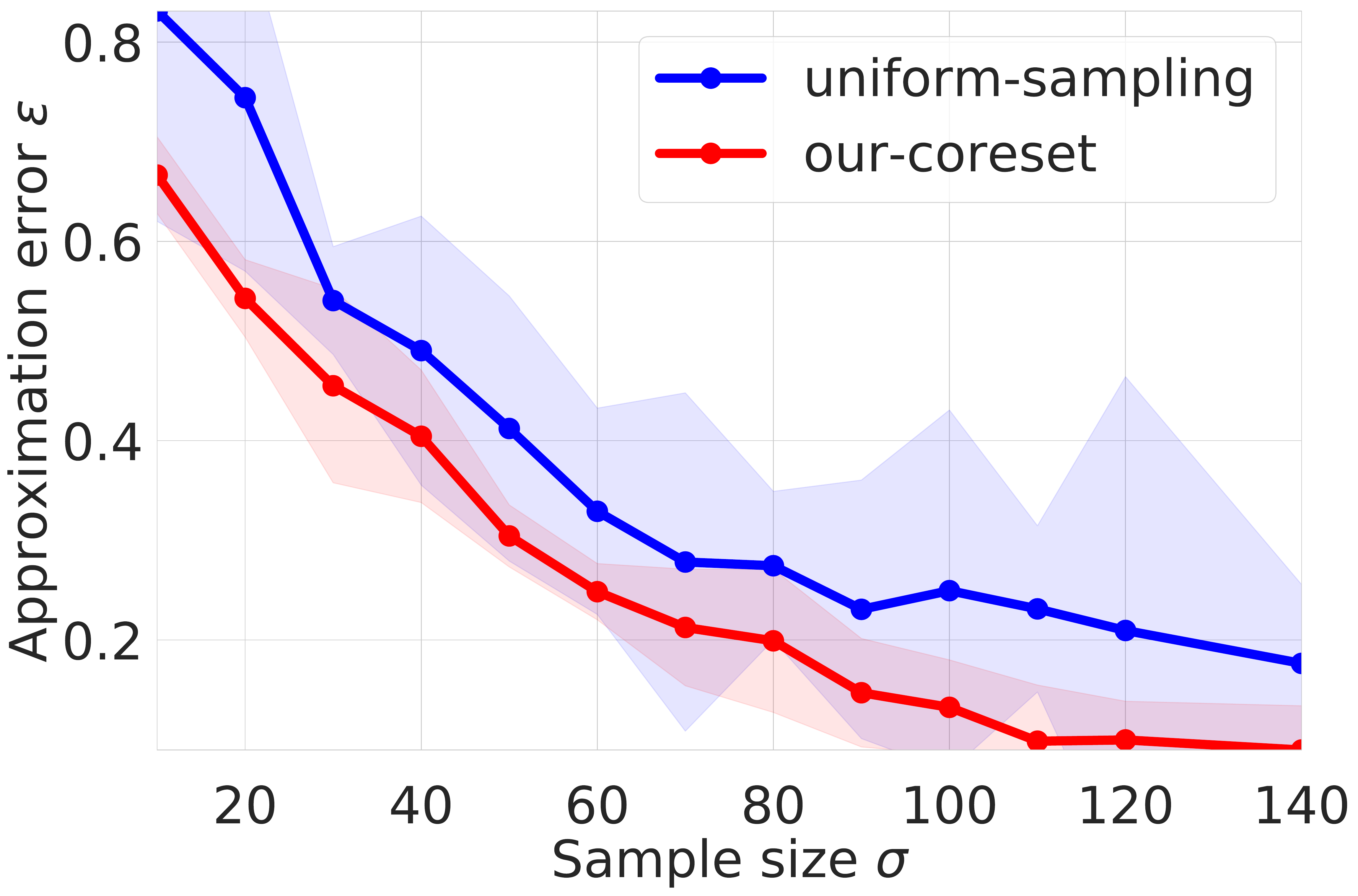}
        \label{fig:dicos2}}
	\subfigure[dataset (ii), $k=8$]{
		\centering
		\includegraphics[width = \s\textwidth]{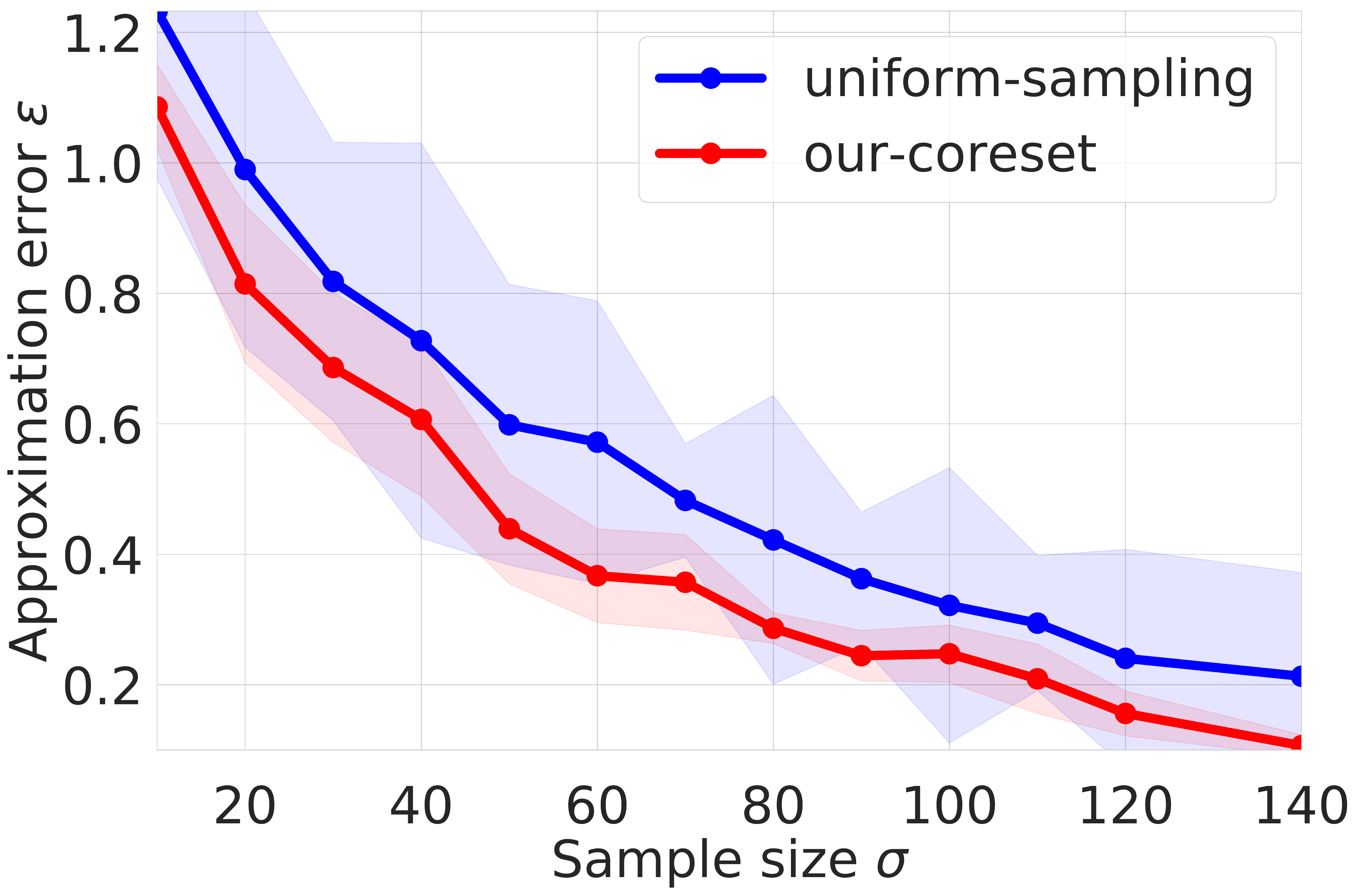}
        \label{fig:dicos3}}
    \subfigure[dataset (ii), $k=10$]{
		\centering
		\includegraphics[width = \s\textwidth]{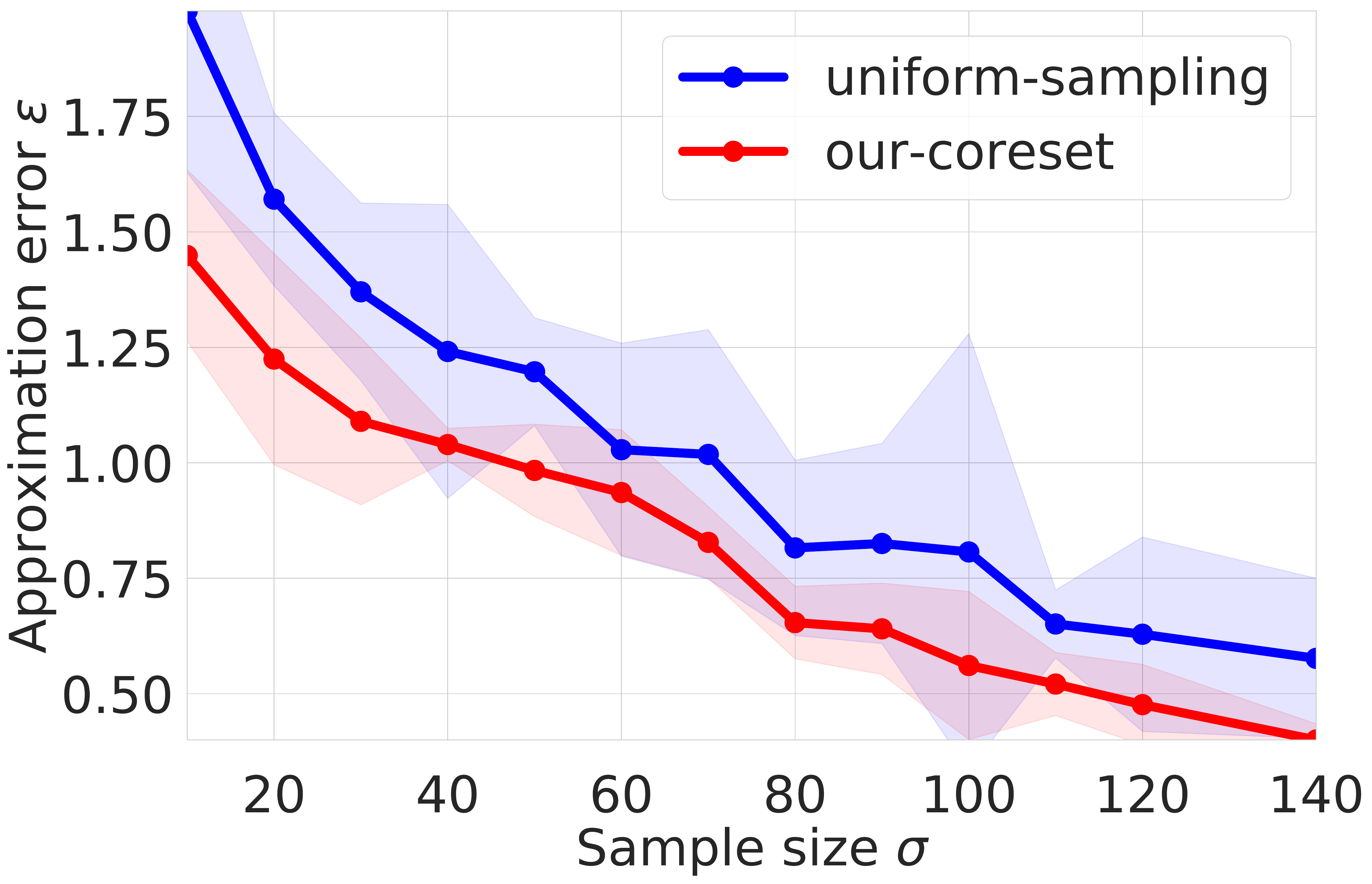}
        \label{fig:dicos4}}
    \subfigure[Optimal sets-mean]{
		\centering
		\includegraphics[width=0.22\textwidth]{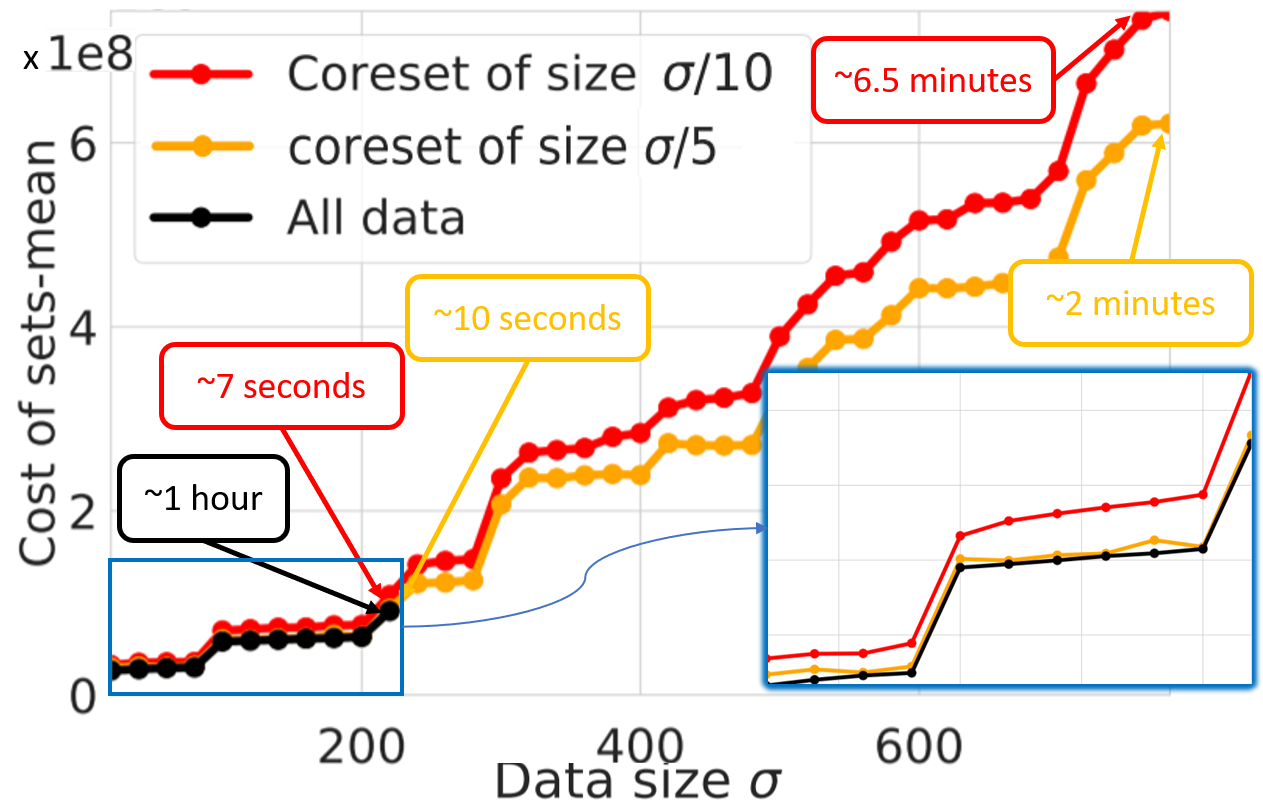}
		\label{fig:optimalTimeComp}}
    \subfigure[$\set{P} :=$ dataset (i)]{
		\centering
		\includegraphics[width = 0.225\textwidth]{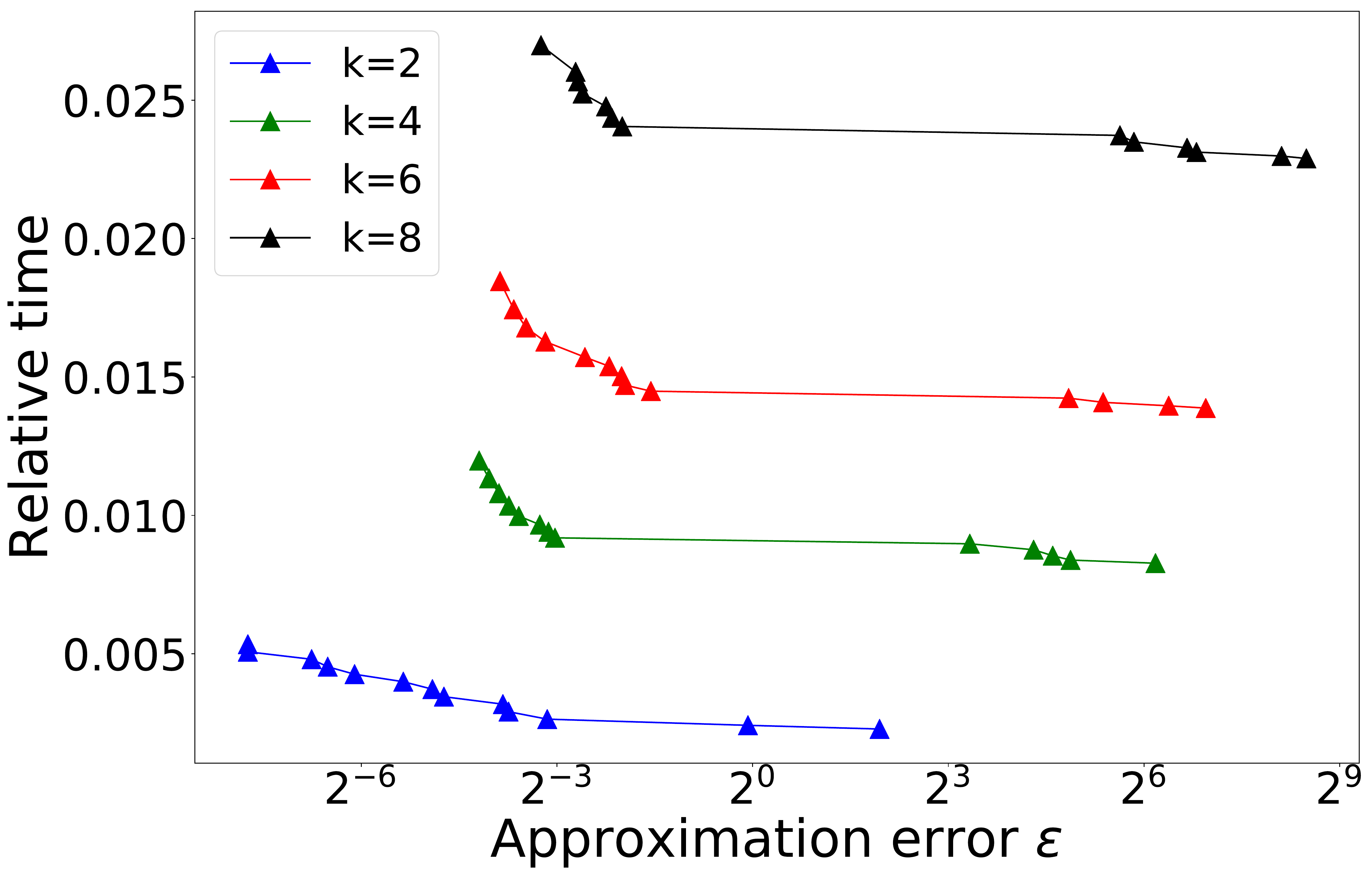}
        \label{fig:time1}}
    \subfigure[$\set{P} :=$ dataset (ii)]{
		\centering
		\includegraphics[width = 0.225\textwidth]{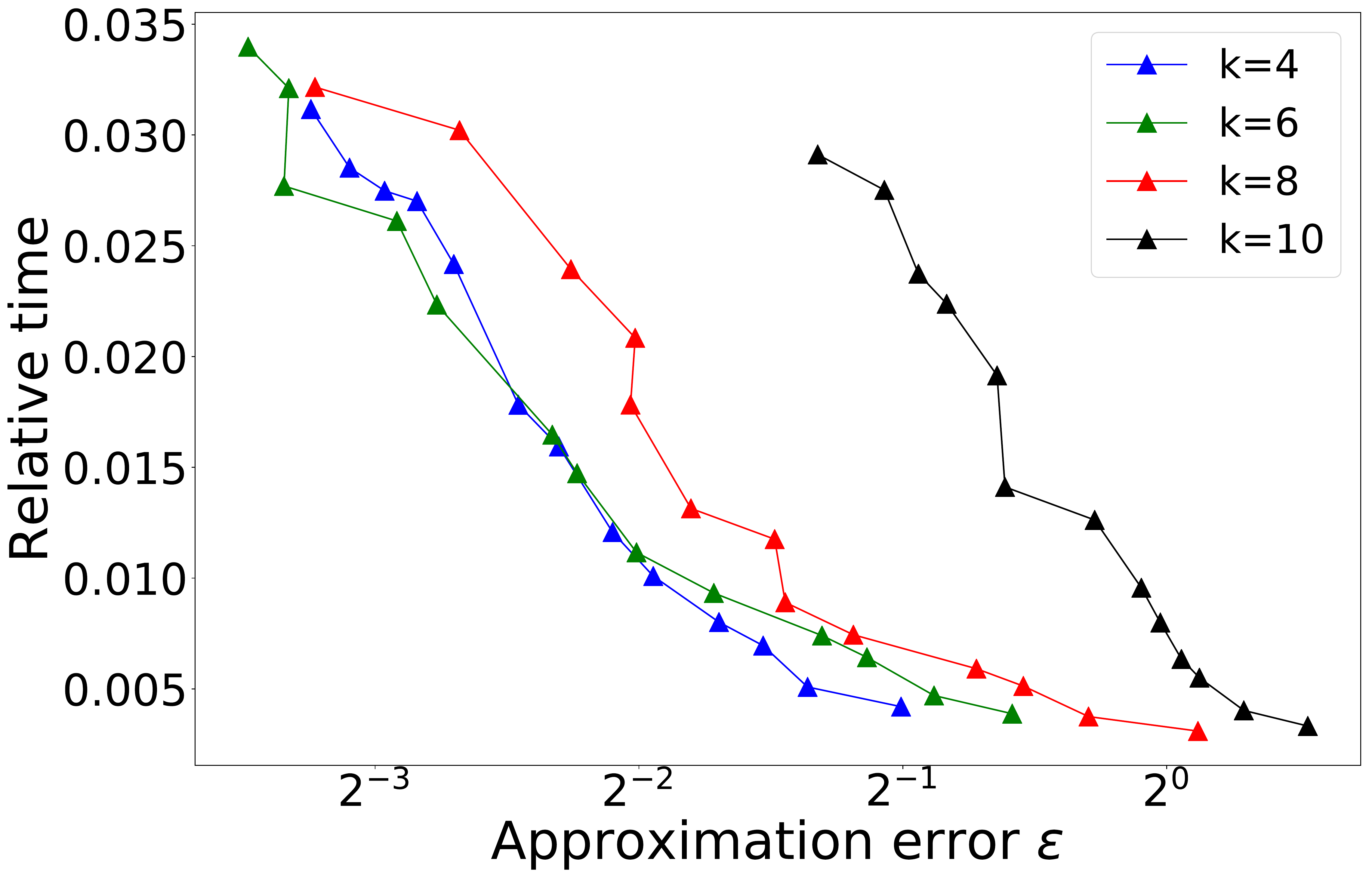}
        \label{fig:time2}}
    \subfigure[$r=10^6$]{
		\centering
		\includegraphics[width = 0.225\textwidth]{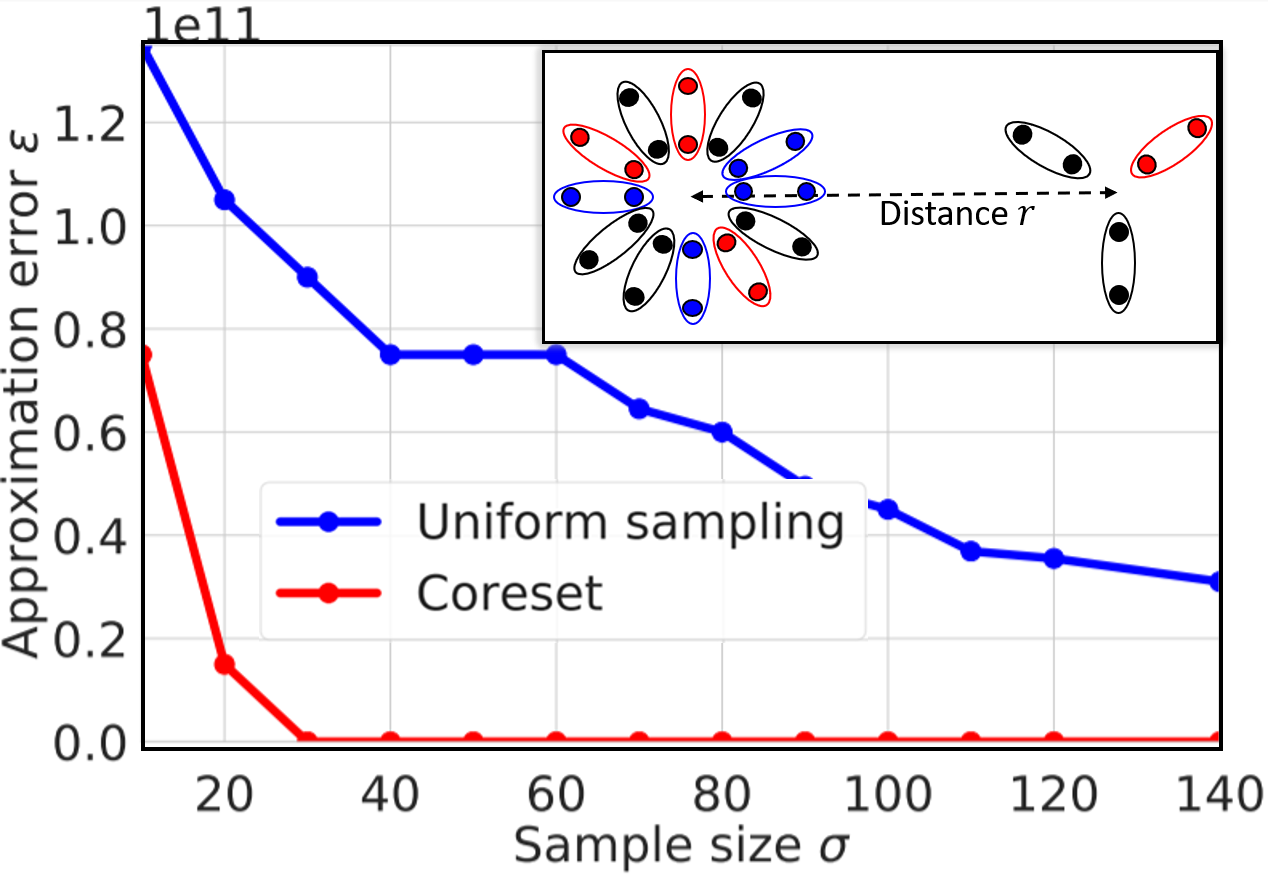}
        \label{fig:synth1}}
    \subfigure[]{
		\centering
		\includegraphics[width = 0.23\textwidth]{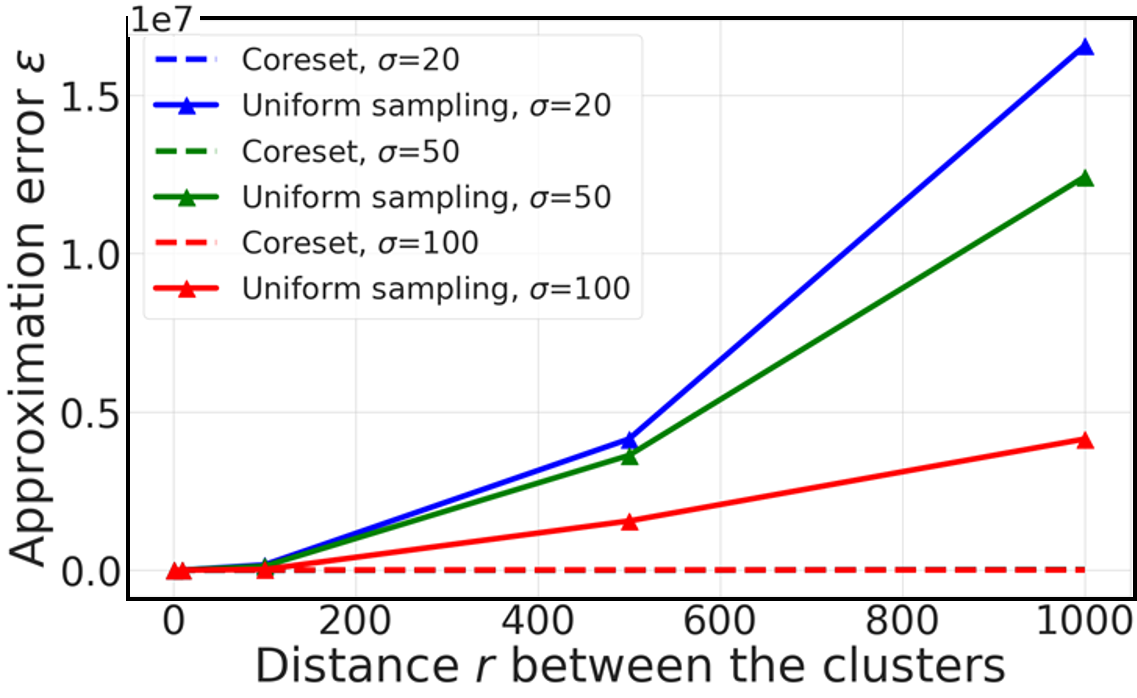}
        \label{fig:synth2}}
    \caption{\textbf{Experimental results.} Exact details are provided in Section~\ref{sec:ER}.
    }
    \label{fig:dicosRes}
\end{figure*}

We implemented our coreset construction, as well as different sets-$k$-mean solvers. In this section we evaluate their empirical performance. Open source code for future research can be downloaded from~\cite{opencode}.

\newcommand{\core}{\texttt{our-coreset}}
\newcommand{\exact}{\texttt{exact-mean}}
\newcommand{\approxalg}{\texttt{approx-mean}}
\newcommand{\km}{\texttt{$k$-means}}
\newcommand{\rand}{\texttt{uniform-sampling}}

\textbf{Theory-implementation gaps. }While Theorem~\ref{cor:PTASKmeans} suggests a polynomial time solution for sets-$k$-means in $\REAL^d$, it is impractical even for $k=1$, $d=2$ and $n=10$ sets. Moreover, its implementation seems extremely complicated and numerically unstable. Instead, we suggest a simple algorithm $\exact$ for computing the sets-mean; see Fig.~\ref{fig:setsVoronoi} in Section~\ref{IA}. Its main components are Voronoi diagram~\cite{aurenhammer1991voronoi} and hyperplanes arrangement that were implemented in Sage~\cite{sagemath}. For $k\geq 2$ we use Expectation-Maximization (EM) heuristic, which is a generalization of the well known Lloyd algorithm~\cite{lloyd1982least}, as commonly used in $k$-means and its variants~\cite{marom2019k, lucic2017training}.

\textbf{Implementations. }Four algorithms were implemented: \textbf{(i): }$\core(\Set{P},\sigma)$: the coreset construction from Algorithm~\ref{alg:wrapper} for a given $(n,m)$-set $\set{P}$, an arbitrary given loss function $\Dt$ that satisfies Definition~\ref{def:distSets} and a sample size of $\abs{\set{S}} = \sigma$ at Line~\ref{line:randomSample} of Algorithm~\ref{alg:wrapper}. \textbf{(ii): }$\rand(\set{P},\sigma)$: outputs a uniform random sample $\set{S} \subseteq\set{P}$ of size $\abs{\set{S}} = \sigma$.
\textbf{(iii): }$\exact(\set{P})$: returns the exact (optimal) sets-mean ($k=1$) of a given set $\set{P}$ of sets in $\REAL^d$ as in the previous paragraph.
\textbf{(iv): }$\km(\set{P},k)$: generalization of the Lloyd $k$-means heuristic~\cite{lloyd1982least} that aims to compute the sets-$k$-mean of $\set{P}$ via EM; see implementation details in Section~\ref{IA}.

\textbf{Software/Hardware. }The algorithms were implemented in Python 3.7.3 using Sage 9.0~\cite{sagemath} as explained above on a Lenovo Z70 laptop with an Intel i7-5500U CPU @ 2.40GHZ and 16GB RAM.

\textbf{Datasets. } \textbf{(i): }The LEHD Origin-Destination Employment Statistics (LODES)~\cite{lodes}. It contains information about people that live and work at the united states. We pick a sample of $n=10,000$ and their $m=2$ home+work addresses, called Residence+Workplace Census Block Code. Each address is converted to a pair $(x,y)$ of $d=2$ doubles. As in Fig.~\ref{fig:facilityLoc}, our goal was to compute the sets-$k$-mean (facilities) of these $n$ pairs of addresses.

\textbf{(ii): }The Reuters-21578 benchmark corpus~\cite{bird2009natural}. It contains $n=10,788$ records that corresponds to $n$ Reuters newspapers. Each newspaper is represented as a ``bag of bag of words" of its $m\in [3,15]$ paragraphs in high dimensional-space; see Fig.~\ref{fig:documentClustering}. Handling sets of different sizes is also supported; see details in Section~\ref{IA}.
We reduce the dimension of the union of these ($3n$ to $15n$) vectors to $d=15$ using LSA~\cite{landauer2013handbook}. The goal was to cluster those $n$ documents (sets of paragraphs) into $k$ topics; see Fig~\ref{fig:documentClustering}.

\textbf{(iii): }Synthetic dataset. We drew a circle of radius $1$, centered at the origin of $\REAL^2$ and then picked $n_1=9900$ points evenly (uniformly) distributed on this circle. For each of these $n_1$ points, we paired a point in the same direction but of distance $30$ from the origin. This resulted in $n_1$ pairs of points. We repeat this for another circle of radius $1$ that is centered at $(r,0)$, for multiple values of $r$, and constructed $n_2=100$ points similarly; see top of Fig.~\ref{fig:synth1}.

\textbf{Experiment (i) }We ran $\set{S}_1(\sigma):=\core(\set{P},\sigma)$ and $\set{S}_2(\sigma):=\rand(\set{P},\sigma)$ on each of the datasets for different values of sample size $\sigma$. Next, we computed the corresponding sets-$k$-means $C_1(\sigma), C_2(\sigma)$ and $C_3$ heuristically using Algorithm (iv). We denote the corresponding computation times in seconds by $t_1(\sigma)$, $t_2(\sigma)$ and $t_3$, respectively. The corresponding costs of $C_1(\sigma)$ and $C_2(\sigma)$ were evaluated by computing the \emph{approximation error}, for $i \in \br{1,2}$, as $\varepsilon_i(\sigma) = \frac{\abs{\sum_{P\in \set{P}}\Dt(P,C_3)-\sum_{P\in \set{P}}\Dt(P,C_i(\sigma))}}{\sum_{P\in \set{P}}\Dt(P,C_3)}$.

\textbf{Results (i). }The approximation errors on the pair of real-world datasets are shown in Fig.~\ref{fig:HW1}--~\ref{fig:dicos4}. Fig~\ref{fig:time1}--~\ref{fig:time2} show relative time $t_1(\sigma)/t_3$ ($y$-axis) as a function of $\varepsilon := \varepsilon_1(\sigma)$ ($x$-axis), for $\sigma = 20,30,\ldots,140$.
The approximation errors are shown for the synthetic dataset, either for different increasing $\sigma$ in Fig.~\ref{fig:synth1} or $r$ values in Fig.~\ref{fig:synth2}.

\textbf{Experiment (ii). } We uniformly sampled $800$ rows from the LEHD Dataset (i). Let $\set{P}(\sigma)$ denote the first $\sigma$ points in this sample, for $\sigma=20,40,60,\ldots, 800$. For each such set $\set{P}(\sigma)$ we computed two different size coresets $\set{S}_1(\sigma) := \core(\set{P}(\sigma),\sigma/10)$ and $\set{S}_2(\sigma) := \core(\set{P}(\sigma),\sigma/5)$. We then applied Algorithm (iii) that computes the optimal sets-mean $C_1(\sigma)$, $C_2(\sigma)$ and $C_3(\sigma)$ on $\set{S}_1(\sigma)$, $\set{S}_2(\sigma)$ and the full data $\set{P}(\sigma)$, respectively.

\textbf{Results (ii). }Fig~\ref{fig:optimalTimeComp} shows the cost of $C_i(\sigma)$ ($y$-axis) as a function of $\sigma$ ($x$-axis), for $i\in \br{1,2,3}$.

\textbf{Discussion. }As common in the coreset literature, we see that the approximation errors are significantly smaller than the pessimistic worst-case bounds.
In all the experiments the coreset yields smaller error compared to uniform sampling. When running exact algorithms on the coreset, the error is close to zero while the running time is reduced from hours to seconds as shown in Fig~\ref{fig:optimalTimeComp}.
The running time is faster by a factor of tens to hundreds using the coresets, in the price of an error between $1/64$ to $1/2$ as shown in Fig.~\ref{fig:time1}--\ref{fig:time2}.

\begin{figure}
  \centering
  \includegraphics[width=0.40\textwidth]{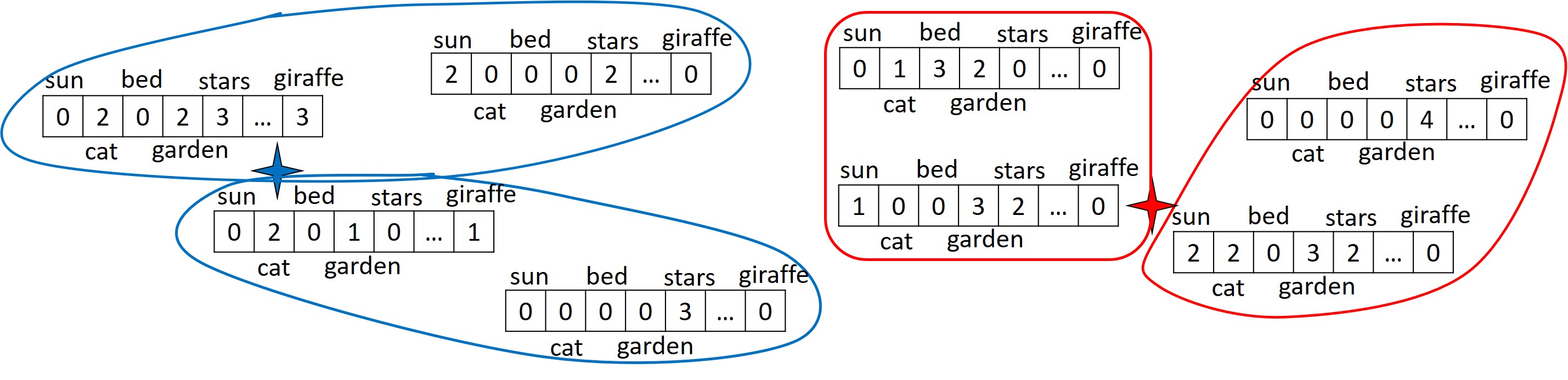}
  \caption{\textbf{bag of bag of words}: A set of $n=4$ documents, each has $m=2$ paragraphs. Each paragraph is represented by a bag of words and each document is represented by the set of $m$ vectors of its paragraphs. The sets-$2$-means are presented (blue / red stars).}
  \label{fig:documentClustering}
\end{figure}


\section{Conclusions and Open Problems}
This paper suggests coresets and near-linear time solutions for clustering of input sets such as the sets-$k$-means. Natural open problems include relaxation to convex optimization, handling other distance functions between sets e.g. max distance, handling infinite sets / shapes (triangles, circles, etc.) and continuous distributions (e.g. $n$ Gaussians).
We hope that this paper is only the first step toward a long line of research that include solutions to the above problems.

\bibliography{refs}
\bibliographystyle{icml2020}


\newpage
\clearpage
\appendix

\section{The Combinatorial Complexity of the Sets Clustering}\label{sec:VC}
The following definition of a query space encapsulates all the ingredients required to formally define an optimization problem.
\begin{definition} [Query space; see Definition 4.2 in~\cite{braverman2016new}] \label{def:querySpace}
Let $\Pset$ be a set called input set. Let $Q$ be be a (possibly infinite) set called query set. Let $f:P\times Q \to \REAL$ be a cost function. The tuple $(\Pset,Q,f)$ is called a \emph{query space}. A \emph{sets clustering query space} is a query space $(\set{P},Q,f)$ where $\set{P}$ is an $(n,m)$-set, $Q$ is the set $\set{X}_k$, and $f = \Dt$; see Section~\ref{sec:problemStatement}.
\end{definition}

In what follows we define some measure of combinatorial complexity for a query space.
\begin{definition} [Definition 4.5 in~\cite{braverman2016new}] \label{def:VC}
For a query space $(\Pset,Q,f)$, a query $C\in Q$ and $r\in [0,\infty)$ we define
\[
\mathrm{range}(\Pset,C,r) = \br{P\in \Pset \mid f(P,c) \leq r}.
\]
Let $\rangess(\Pset,Q,f) = \br{\mathrm{range} \left(\Pset, C, r \right) \middle| C \in Q, r \geq 0}$,
the VC-dimension of $(P,\rangess(\Pset,Q,f))$ is the smallest integer $d'$ such that for every $\set{H} \subseteq \Pset$ we have
\[
\left| \br{\mathrm{range}(C,r) \mid C\in Q, r\in [0,\infty)} \right| \leq |\set{H}|^{d'}.
\]
The dimension of the query space $(\Pset,Q,f)$ is the VC-dimension of $(P,\rangess(\Pset,Q,f))$.
\end{definition}

\begin{lemma}[Variant of Theorem 8.4,~\cite{anthony2009neural}]
\label{lem:VCBoundOrig}
Suppose $h$ is a function from $\REAL^d \times \REAL^n$ to $\br{0,1}$ and let
\[
H = \br{h(a,x) \mid  a \in \REAL^d, x \in \REAL^n}
\]
be the class determined by $h$. Suppose that $h$ can be computed by an algorithm that takes as an input a pair $(a,x) \in \REAL^d \times \REAL^n$ and returns $h(a,x)$ after no more than $t$ operations of the following types:
\begin{itemize}
    \item the arithmetic operations $+,-,\times,$ and $/$ on real numbers,
    \item jumps conditioned on $>, \geq, <, \leq, =$, and $\neq$ comparisons of real numbers, and
    \item outputs $0$ or $1$.
\end{itemize}

Then the $VC$-dimension of $H$ is $O\left(d t \right)$.
\end{lemma}

We now bound the dimension of a query space $(\set{P},\set{X}_k,\Dt)$ as in Definition~\ref{def:VC}.
\begin{lemma}
\label{lem:VCBound}
Let $(\set{P},\set{X}_k,\Dt)$ be a sets clustering query space; see Definition~\ref{def:querySpace}. Then the dimension $d'$ of $(\set{P},\set{X}_k,\Dt)$ is bounded by $\in O(md^2k^2)$.
\end{lemma}
\begin{proof}
For $P \in \set{P}$, $C \in \set{X}_k$, and $r \in \REAL$, let $h_P(C,r) = 1$ if  $\Dt(P,C) \geq r$ and $0$ otherwise. Then we observe that the $VC$-dimension of the class of functions $H = \br{h_{P}:\set{X}_k\times \REAL \to [0,\infty) \mid P \in \set{P}}$ in Lemma~\ref{lem:VCBoundOrig} is equivalent to the dimension $d'$ of the given query space. Therefore, we now show that the $VC$-dimension of $H$ is bounded by $O(md^2k^2)$.

Note that it takes $t=O(mdk)$ arithmetic operations to evaluate $h_P(C,r)$. Furthermore, any element in $\set{X}_k\times \REAL$ can be represented as a vector in $(dk+1)$-dimensional space. Hence by Lemma~\ref{lem:VCBoundOrig}, the $VC$-dimension of $H$ is $O(dk\cdot mdk) = O(md^2k^2)$.
\end{proof}

\section{Main theorems with full proof}

\subsection{Proof of Lemma~\ref{lem:sens}}

\begin{lemma}
\label{helperLemma}
Let $k \geq 1$ be an integer, $A,B \subseteq \M$ and $C \in \set{X}_k$. If $\Dt(A \cup B, C) \neq \Dt(B,C)$ then $\Dt(A\cup B,C) = \Dt(A,C)$.
\end{lemma}

\begin{proof}
By definition, $\Dt(A \cup B, C) = \min \br{\Dt(A, C),\Dt(B, C)}$. By the assumption of the lemma, $\Dt(A \cup B, C) \neq \Dt(B,C)$. Therefore, $\Dt(A \cup B, C) = \Dt(A, C)$
\end{proof}

\begin{lemma}
\label{onePointSwap}
Let $A = \br{a_1,\cdots,a_n} \subseteq \M$ and put $b \in \M$. Let $B = (A \setminus\br{a_1}) \cup \br{b} = \br{b,a_2,\cdots,a_n} \subseteq \M$. Then for every $C \in \set{X}_k$ we have that
\[
\Dt(A,C) \leq \rho\left(\Dt(B,C) + \Dt(a_1,b)\right).
\]
\end{lemma}

\begin{proof}
By definition, we have that
\[
\begin{split}
&\Dt(A,C) \\
&= \min \br{\Dt(a_1,C), \Dt(A\setminus\br{a_1},C)}\\
& \leq \min \left\lbrace \rho\left(\Dt(a_1,b) + \Dt(b,C)\right), \right.\\ & \quad\quad\quad\quad\left. \Dt(A\setminus\br{a_1},C) \right\rbrace\\
& \leq \min \left\lbrace\rho\left(\Dt(a_1,b) + \Dt(b,C)\right),\right.\\ &\quad\quad\quad\quad\left.\rho\left(\Dt(A\setminus\br{a_1},C) + \Dt(a_1,b)\right)\right\rbrace\\
& \leq \rho \min \br{\Dt(b,C), \Dt(A\setminus\br{a_1},C)} + \rho\Dt(a_1,b)\\
& = \rho\Dt(B,C) +\rho \Dt(a_1,b),
\end{split}
\]
where the first inequality is by the weak triangle inequality by Lemma~\ref{lem:weakTriangle}, and the last derivation is by the definition of $B$.
\end{proof}

\renewcommand{\thesection}{\arabic{subsection}}
\setcounter{subsection}{4}
\setcounter{theorem}{0}

\sens*
\begin{proof}
In what follows, we use the variables and notations from Algorithm~\ref{alg:RobustMedSets}.
Put $P \in \set{P}^m$, $i \in [m]$, and consider the $i$th iteration of the ``for'' loop at Line~\ref{alg:forLine} of Algorithm~\ref{alg:RobustMedSets}. Put $C \in \set{X}_k$.

Let
\[
\overline{\set{P}}^{i-1} =\br{Q \in \set{P}^{i-1} \mid \Dt(\proj(Q,\set{B}^{i-1}),C) =\Dt(\set{B}^{i-1} ,C)}
\]
be the union of sets $Q \in \set{P}^{i-1}$ whose closest point to the query $C$ after the projection on $\set{B}^{i-1}$ is one of the points of $\set{B}^{i-1}$.
First we prove that
\begin{align}
\frac{\Dt(\proj(P,\set{B}^{i-1}),C)}{\sum\limits_{Q \in \set{P}^{i-1}} \Dt(\proj(Q,\set{B}^{i-1}),C)}&  \leq 3\rho^2 \frac{\Dt(\proj(P,\set{B}^{i}),C)}{\sum\limits_{Q \in \set{P}^{i}} \Dt(\proj(Q,\set{B}^{i}),C)} \nonumber\\ &+\frac{4\rho}{\abs{\set{P}^i}} \label{eq:inductionStep}
\end{align}
by the following case analysis: \textbf{(i) }$\abs{\overline{\set{P}}^{i-1} }\geq \frac{\abs{\Set{P}^{i-1}}}{2}$, i.e., more than half the sets satisfy that their closest point to $C$ is amongst their projected points onto $\set{B}^{i-1}$, and \textbf{(ii) } Otherwise, i.e., $\abs{\overline{\set{P}}^{i-1} } < \frac{\abs{\Set{P}^{i-1}}}{2}$.

\textbf{Case (i): }$\abs{\overline{\set{P}}^{i-1} }\geq \frac{\abs{\Set{P}^{i-1}}}{2}$.
By Line~\ref{newPi} we have
\begin{equation} \label{eqsubsets}
\displaystyle{\set{P}^i \subseteq \set{P}^{i-1} \subseteq \cdots \subseteq  \set{P}^0=\set{P}}.
\end{equation}

Therefore,
\begin{align}
& \sum\limits_{Q \in \set{P}^{i-1}} \Dt(\proj(Q,\set{B}^{i-1}),C) \geq \sum_{Q\in \overline{\set{P}}^{i-1}} \Dt(\proj(Q,\set{B}^{i-1}),C) \label{eqproj1}\\
& = \sum_{Q\in \overline{\set{P}}^{i-1}} \Dt(\set{B}^{i-1},C)
\geq \frac{\abs{\set{P}^{i-1}}}{2} \Dt(\set{B}^{i-1} ,C), \label{eqproj3}
\end{align}
where~\eqref{eqproj1} holds since $\overline{\set{P}}^{i-1} \subseteq \set{P}^{i-1}$, the first derivation in~\eqref{eqproj3} is by the definition of $\overline{\set{P}}^{i-1}$, and the second derivation in~\eqref{eqproj3} is by the assumption of Case (i).
This proves~\eqref{eq:inductionStep} for Case (i) as
\begin{align}
& \frac{\Dt(\proj(P,\set{B}^{i-1}),C)}{\sum\limits_{Q \in \set{P}^{i-1}} \Dt(\proj(Q,\set{B}^{i-1}),C)}
\leq \frac{\Dt(\set{B}^{i-1} ,C)}{\sum\limits_{Q \in \set{P}^{i-1}} \Dt(\proj(Q,\set{B}^{i-1}),C)} \nonumber \\
& \leq \frac{\Dt(\set{B}^{i-1} ,C)}{\frac{\abs{\set{P}^{i-1}}}{2} \Dt(\set{B}^{i-1},C)}
= \frac{2}{\abs{\set{P}^{i-1}}} \leq \frac{2}{\abs{\set{P}^{i}}}, \label{eqFinalCase1}
\end{align}
where the first inequality holds since $\set{B}^{i-1} \subseteq \proj(P,\set{B}^{i-1})$ by Definition~\ref{def:proj}, and the second inequality is by~\eqref{eqproj3}.%

\textbf{Case (ii):} $\abs{\overline{\set{P}}^{i-1} }< \frac{\abs{\Set{P}^{i-1}}}{2}$.
Let $\gamma = 1/(2k)$.
Let ${\hat{\set{P}}}^{i-1}$, $b^i$ and $\set{P}^i$ be as defined in Lines~\ref{alg:LinePHati},~\ref{alg:LineBi}, and~\ref{newPi} respectively, and identify $\set{B}^{i-1} = \br{b^1,\cdots,b^{i-1}}$ for $i \geq 2$ or $\set{B}^{i-1} = \emptyset$ for $i=1$. Let
\begin{equation} \label{eqOPTi}
OPT_i = \min\limits_{C'\in \set{X}_k} \sum\limits_{\hat{P}\in \closest(\Set{\hat{P}}^{i-1},C',1/2)} \Dt(\hat{P},C').
\end{equation}

For every $Q \in \set{P}^{i-1}$, substituting $A = \notail(Q,\set{B}^{i-1})$ and $B = \set{B}^{i-1}$ in Lemma~\ref{helperLemma} proves that
\begin{equation} \label{eqSub}
\begin{split}
& \left\lbrace \mkern-17mu\begin{array}{c!{\vline width 0.6pt}c}
\begin{array}{c}
     Q \\
     \in \set{P}^{i-1}
\end{array} & \begin{array}{c}
     \Dt\big( \notail(Q,\set{B}^{i-1}) \cup \set{B}^{i-1},C\big)  \\
     \neq \Dt(\set{B}^{i-1},C)
\end{array} \end{array} \mkern-17mu\right\rbrace \\
& \subseteq
\left\lbrace \mkern-17mu\begin{array}{c!{\vline width 0.6pt}c}
     \begin{array}{c}
          Q \\
          \in \set{P}^{i-1}
     \end{array} &   \begin{array}{c} \Dt\big( \notail(Q,\set{B}^{i-1}) \cup \set{B}^{i-1},C\big) \\ =\Dt(\notail(Q,\set{B}^{i-1}),C) \end{array}
\end{array}
\mkern-17mu \right\rbrace
\end{split}
\end{equation}

We now obtain that
\begin{align}
& \abs{ \left\lbrace \mkern-17mu\begin{array}{c!{\vline width 0.6pt}c}
    \begin{array}{c}
         Q \\
         \in \set{P}^{i-1}
    \end{array} & \begin{array}{c}
    \Dt\left(\proj(Q,\set{B}^{i-1}),C\right) \\
    = \Dt\left(\notail(Q,\set{B}^{i-1}),C\right)
    \end{array}
\end{array} \mkern-17mu \right\rbrace} \nonumber\\
& = \abs{ \left\lbrace \mkern-17mu \begin{array}{c!{\vline width 0.6pt}c}
\begin{array}{c}
     Q \\
     \in \set{P}^{i-1}
\end{array} & \begin{array}{c}
 \Dt\big(  \notail(Q,\set{B}^{i-1}) \cup \set{B}^{i-1},C\big) \\
 = \Dt\left(\notail(Q,\set{B}^{i-1}),C \right)
\end{array}
\end{array} \mkern-17mu\right\rbrace} \label{eqsize1}\\
& \geq \abs{ \left\lbrace \mkern-17mu \begin{array}{c!{\vline width 0.6pt}c}
\begin{array}{c}
     Q \\
     \in \set{P}^{i-1}
\end{array} & \begin{array}{c}
\Dt\big(  \notail(Q,\set{B}^{i-1}) \cup \set{B}^{i-1},C\big) \\
\neq \Dt\left(\set{B}^{i-1},C \right)
\end{array}
\end{array} \mkern-17mu\right\rbrace} \label{eqsize3}\\
& = \abs{\left\lbrace \mkern-17mu\begin{array}{c!{\vline width 0.6pt}c} \begin{array}{c}
     Q
     \in \set{P}^{i-1}
\end{array} & \begin{array}{c}
\Dt\left(\proj(Q,\set{B}^{i-1}),C \right) \\
\neq \Dt(\set{B}^{i-1},C)
\end{array}
\end{array} \mkern-17mu\right\rbrace} \label{eqsize32}\\
& = \big|\set{P}^{i-1} \setminus \overline{\set{P}}^{i-1}\big| \geq \frac{\abs{\set{P}^{i-1}}}{2}, \label{eqsize4}
\end{align}
where~\eqref{eqsize1} and~\eqref{eqsize32} is by substituting $\set{P} = Q$ and $\set{B} = \set{B}^{i-1}$ in Definition~\ref{def:proj}, \eqref{eqsize3} is by~\eqref{eqSub}, the first derivation in~\eqref{eqsize4} is by the definitions of $\set{P}^{i-1}$ and $\overline{\set{P}}^{i-1}$, and the last inequality is by the assumption of Case (ii).

Recall that by Line~\ref{alg:LinePHati},
\[
\set{\hat{P}}^{i-1} = \br{ \notail(Q,\set{B}^{i-1}) \mid  Q \in \set{P}^{i-1}},
\]
and let
\[
Z = \br{Q \in \set{P}^{i-1} \mid \notail(Q,\set{B}^{i-1}) \in \closest(\Set{\hat{P}}^{i-1},C,1/2)}.
\]

Since $Z$ contains the $|Z| \leq \frac{|\set{P}^{i-1}|}{2}$ sets $Q \in \set{P}^{i-1}$ with the smallest $\Dt(\notail(Q,\set{B}^{i-1}),C)$, for any set $Z' \subseteq \set{P}^{i-1}$ such that $|Z'| \geq \frac{|\set{P}^{i-1}|}{2}$, we have
\begin{equation}\label{eqZProp}
\begin{split}
\sum\limits_{Q \in Z} \Dt(\notail(Q,\set{B}^{i-1}),C)
&\leq \sum\limits_{Q \in Z'} \Dt(\notail(Q,\set{B}^{i-1}),C).
\end{split}
\end{equation}

By the assumption of Case (ii),
\begin{equation}\label{eqAssii}
\abs{ \set{P}^{i-1} \setminus \overline{\set{P}}^{i-1}} \geq \frac{\abs{\set{P}^{i-1}}}{2},
\end{equation}
and by the definition of $Z$, we have
\begin{equation} \label{eqZisClosest}
\br{\notail(Q,\set{B}^{i-1}) \mid Q \in Z} = \closest(\Set{\hat{P}}^{i-1},C,1/2).
\end{equation}
Therefore,
\begin{equation}\label{eqClosestIsBest}
\begin{split}
&\sum\limits_{\hat{Q}\in \closest(\Set{\hat{P}}^{i-1},C,1/2)} \Dt(\hat{Q},C) \\
&= \sum\limits_{Q \in Z} \Dt(\notail(Q,\set{B}^{i-1}),C)\\
& \leq \sum\limits_{Q \in \set{P}^{i-1} \setminus \overline{\set{P}}^{i-1}} \Dt(\notail(Q,\set{B}^{i-1}),C),
\end{split}
\end{equation}
where the first derivation is by~\eqref{eqZisClosest} and the last derivation is by substituting $Z' = \set{P}^{i-1} \setminus \overline{\set{P}}^{i-1}$ in~\eqref{eqZProp}. By the definitions of $\set{P}^{i-1}$ and $\overline{\set{P}}^{i-1}$, for every $Q \in \set{P}^{i-1} \setminus \overline{\set{P}}^{i-1}$, we have
\begin{equation} \label{eqTEqalProj}
\Dt(\proj(Q,\set{B}^{i-1}),C) = \Dt(\notail(Q,\set{B}^{i-1}),C).
\end{equation}
Hence,
\begin{align}
\OPT_i &\leq \sum\limits_{\hat{Q}\in \closest(\Set{\hat{P}}^{i-1},C,1/2)} \Dt(\hat{Q},C) \label{z1}
\\&\leq \sum\limits_{Q \in \set{P}^{i-1} \setminus \overline{\set{P}}^{i-1}} \Dt(\notail(Q,\set{B}^{i-1}),C)\label{z2}
\\& = \sum\limits_{Q \in \set{P}^{i-1} \setminus \overline{\set{P}}^{i-1}} \Dt(\proj(Q,\set{B}^{i-1}),C)\label{z4}
\\& \leq \sum\limits_{Q \in \set{P}^{i-1}} \Dt(\proj(Q,\set{B}^{i-1}),C),\label{z5}
\end{align}
where~\eqref{z1} holds by the definition of $OPT_i$, \eqref{z2} is by~\eqref{eqClosestIsBest}, and~\eqref{z4} is by~\eqref{eqTEqalProj}.

Recall that $P \in \set{P}^m$, identify
\[
\closepoints(P,\set{B}^m) = \br{(\hat{p}_1,\hat{b}_1),\cdots,(\hat{p}_m,\hat{b}_m)},
\]
as in Definition~\ref{def:proj} (i). Also by Definition~\ref{def:proj}, for every $i\in [m]$ we have
\begin{equation}
\label{eq:obsr1}
\begin{split}
&\Dt(\notail(P,\set{B}^{i-1}),\hat{b}_i) \\
&= \Dt(P \setminus \br{\hat{p}_1,\cdots,\hat{p}_{i-1}},\hat{b}_i)\\
&= \Dt(\hat{p}_i,\hat{b}_i).
\end{split}
\end{equation}
Since $P \in \set{P}^i$ and $\gamma = \frac{1}{2k}$, we have by Line~\ref{newPi} that
\begin{equation} \label{eqPMarkov}
\notail(P,\set{B}^{i-1}) \in \closest(\set{\hat{P}}^{i-1}, \br{b^i}, (1-\tau)\gamma/2).
\end{equation}

Observe that in the definition of $\OPT_i$ in~\eqref{eqOPTi}, the largest cluster in every set $C'$ of $k$ centers contains at least $\frac{|\hat{\set{P}}^{i-1}|}{2k} = \gamma |\hat{\set{P}}^{i-1}|$ points by the Pigeonhole Principle. Therefore, since the cost of the closest $(1-\tau)\gamma|\hat{\set{P}}^{i-1}|$ sets for $\hat{b}^i$ is a $2$-approximation for the optimal set of $\gamma |\hat{\set{P}}^{i-1}|$ points, we have
\begin{equation} \label{eqCostbOPT}
\begin{split}
& \sum\limits_{Q \in \closest(\set{\hat{P}}^{i-1},\br{\hat{b}_i},(1-\tau)\gamma)}\Dt(Q,\hat{b}_i)\\ &\leq 2 \min_{\br{b} \in \set{X}_1}\sum\limits_{Q \in \closest(\set{\hat{P}}^{i-1},\br{b},\gamma)}\Dt(Q,b) \\
&\leq 2 \cdot \OPT_i.
\end{split}
\end{equation}
Therefore,
\begin{align}
&\Dt(\hat{p}_i,\hat{b}_i) = \Dt(\notail(P,\set{B}^{i-1}),\hat{b}_i) \label{eqDpibi1}
\\
&\leq 2\cdot\frac{\sum\limits_{Q \in \closest(\set{\hat{P}}^{i-1},\br{\hat{b}_i},(1-\tau)\gamma)}\Dt(Q,\hat{b}_i)}{(1-\tau)\gamma\abs{\set{\hat{P}}^{i-1}}} \label{eqDpibi2} \\
& \leq 2\cdot\frac{\sum\limits_{Q \in \closest(\set{\hat{P}}^{i-1},\br{\hat{b}_i},(1-\tau)\gamma)}\Dt(Q,\hat{b}_i)}{\abs{\set{P}^{i}}} \label{eqDpibi3}\\
&\leq \frac{4 \OPT_i}{\abs{\set{P}^{i}}}, \label{eqRadius}
\end{align}
where~\eqref{eqDpibi1} is by~\eqref{eq:obsr1}, \eqref{eqDpibi2} is by combining Markov's Inequality with~\eqref{eqPMarkov}, \eqref{eqDpibi3} follows since $|\set{P}^i| = \frac{(1-\tau)\gamma}{2}|\set{P}^{i-1}| \leq (1-\tau)\gamma |\set{P}^{i-1}|$, and~\eqref{eqRadius} is by~\eqref{eqCostbOPT}.

Now, since the sets $\proj(P,\set{B}^{i-1})$ and $\proj(P,\set{B}^{i})$ differ by at most one point, i.e.,
\[
\proj(P,\set{B}^{i}) = \left(\proj(P,\set{B}^{i-1}) \setminus \br{\hat{p}_i}\right) \cup \br{\hat{b}_i},
\]
by substituting $A = \proj(P,\set{B}^{i-1})$, and $B = \proj(P,\set{B}^{i})$ in Lemma~\ref{onePointSwap}, we obtain that
\begin{equation} \label{eq1change}
\Dt(\proj(P,\set{B}^{i-1}),C) \leq \rho\Dt(\proj(P,\set{B}^{i}),C) + \rho\Dt(\hat{p}_i, \hat{b}_i).
\end{equation}
By the previous inequality we obtain
\begin{equation} \label{eqtriaProj}
\begin{split}
\frac{\Dt(\proj(P,\set{B}^{i-1}),C)}{\sum\limits_{Q \in \set{P}^{i-1}} \Dt(\proj(Q,\set{B}^{i-1}),C)}& \leq \rho \frac{\Dt(\proj(P,\set{B}^{i}),C)}{\sum\limits_{Q \in \set{P}^{i-1}} \Dt(\proj(Q,\set{B}^{i-1}),C)}\\
& + \rho\frac{\Dt(\hat{p}_i, \hat{b}_i)}{\sum\limits_{Q \in \set{P}^{i-1}} \Dt(\proj(Q,\set{B}^{i-1}),C)}.
\end{split}
\end{equation}
We now bound the rightmost term of~\eqref{eqtriaProj} as
\begin{align}
& \rho\frac{\Dt(\hat{p}_i, \hat{b}_i)}{\sum\limits_{Q \in \set{P}^{i-1}} \Dt(\proj(Q,\set{B}^{i-1}),C)} \leq   \rho\frac{\Dt(\hat{p}_i, \hat{b}_i)}{\OPT_i} \label{zb2} \\
& \leq  \rho\frac{4 \OPT_i}{\abs{\set{P}^{i}} \OPT_i} \label{zb3} = 4\rho\frac{1}{\abs{\set{P}^{i}}},
\end{align}
where~\eqref{zb2} is by~\eqref{z5}, and the first derivation in~\eqref{zb3} is by~\eqref{eqRadius}.

We now bound the middle term of~\eqref{eqtriaProj}.
By identifying $\closepoints(Q,\set{B}^m) = \br{(\hat{q}_1,\hat{b}_1),\cdots,(\hat{q}_m,\hat{b}_m)}$ for every $Q\in \set{P}^i$, we have,
\begin{align}
&\sum\limits_{Q \in \set{P}^{i}} \Dt(\proj(Q,\set{B}^{i}),C) \nonumber\\
&\leq \rho \sum\limits_{Q \in \set{P}^{i}} \Dt(\proj(Q,\set{B}^{i-1}),C) + \rho \sum\limits_{Q \in \set{P}^{i}} \Dt(\hat{q}_i, \hat{b}_i) \label{eq:case2Eq1} \\
&\leq \rho\sum_{Q\in \set{P}^i}\Dt(\proj(Q,\set{B}^{i-1}),C)  +  \rho \abs{\set{P}^i} \frac{2\OPT_i}{\abs{\set{P}^i}}  \label{eq:case2Eq2}\\
&\leq \rho\sum_{Q\in \set{P}^{i-1}}\Dt(\proj(Q,\set{B}^{i-1}),C)  +  2\rho \OPT_i  \label{eq:case2Eq3}\\
&\leq (\rho+2\rho) \sum_{Q\in \set{P}^{i-1}} \Dt(\proj(Q,\set{B}^{i-1}),C),   \label{eq:case2Eq4}
\end{align}
where~\eqref{eq:case2Eq1} follows similarly to~\eqref{eq1change}, \eqref{eq:case2Eq2} holds similarly to~\eqref{eqRadius} for the set $Q$ instead of $P$, \eqref{eq:case2Eq3} holds since $\set{P}^i \subseteq \set{P}^{i-1}$ by~\eqref{eqsubsets} and~\eqref{eq:case2Eq4} is by~\eqref{z5}. Thus, by~\eqref{eq:case2Eq4}, the middle term of~\eqref{eqtriaProj} is bounded by
\begin{align}
\rho \frac{\Dt(\proj(P,\set{B}^{i}),C)}{\sum\limits_{Q \in \set{P}^{i-1}} \Dt(\proj(Q,\set{B}^{i-1}),C)} \leq 3\rho^2 \frac{\Dt(\proj(P,\set{B}^{i}),C)}{\sum\limits_{Q \in \set{P}^{i}} \Dt(\proj(Q,\set{B}^{i}),C)}. \label{denombound}
\end{align}

By combining~\eqref{eqtriaProj},~\eqref{zb3} and~\eqref{denombound}, we get that
\begin{equation}
\label{eqFinalCase2}
\begin{split}
&\frac{\Dt(\proj(P,\set{B}^{i-1}),C)}{\sum\limits_{Q \in \set{P}^{i-1}} \Dt(\proj(Q,\set{B}^{i-1}),C)} \\
&\leq 3\rho^2 \frac{\Dt(\proj(P,\set{B}^{i}),C)}{\sum\limits_{Q \in \set{P}^{i}} \Dt(\proj(Q,\set{B}^{i}),C)} + 4\rho\frac{1}{\abs{\set{P}^{i}}}.
\end{split}
\end{equation}

Now~\eqref{eq:inductionStep} holds by taking the maximum between the bounds of Case (i) in~\eqref{eqFinalCase1}, and the bound of Case (ii) in~\eqref{eqFinalCase2}.

We can now apply~\eqref{eq:inductionStep} recursively over every $i\in [m]$ to obtain that
\begin{align}
\label{eq:final}
&\frac{\Dt(P,C)}{\sum\limits_{Q \in \set{P}} \Dt(Q,C)} = \frac{\Dt(\proj(P,\set{B}^{0}),C)}{\sum\limits_{Q \in \set{P}^{0}} \Dt(\proj(Q,\set{B}^{0}),C)}\\
& \leq (3\rho)^{2m} \frac{\Dt(\proj(P,\set{B}^{m}),C)}{\sum\limits_{Q \in \set{P}^{m}} \Dt(\proj(Q,\set{B}^{m}),C)} + 4\rho\sum\limits_{i \in [m]} \frac{(3\rho^2)^{i-1}}{\abs{\set{P}^i}}.
\end{align}
Also, for every $Q\in \set{P}^m$ observe that $\abs{Q} = \abs{\set{B}^m} = m$, hence
\[
\proj(Q,\set{B}^m) = \set{B}^m= \br{\hat{b}_1,\cdots,\hat{b}_m}.
\]
Thus, for every $Q\in \set{P}^m$ and $C\in \set{X}_k$
\begin{align}
    \Dt(\proj(Q,\set{B}^{m}),C) =\Dt\Big(\br{\hat{b}_1,\cdots,\hat{b}_m} ,C\Big) \label{unif}
\end{align}

Lemma~\ref{lem:sens} now holds as
\begin{align}
\frac{\Dt(P,C)}{\sum\limits_{Q \in \set{P}} \Dt(Q,C)}
&\leq \frac{(3\rho^2)^{m}}{\abs{\set{P}^m}} + 4\rho\sum\limits_{i \in [m]} \frac{(3\rho^2)^{i-1}}{\abs{\set{P}^i}} \label{sensP1}\\
& \leq \frac{(3\rho^2)^{m}}{\abs{\set{P}^m}} + 4\rho\sum\limits_{i \in [m]} \frac{(3\rho^2)^{i-1}}{\abs{\set{P}^m}} \label{sensP2}\\
& \leq \frac{(3\rho^2)^{m}}{\abs{\set{P}^m}} + \frac{4\rho}{\abs{\set{P}^m}} \cdot \frac{(3\rho^2)^{m-1} -1}{(3\rho^2) -1} \label{sensP3}\\
& \leq \frac{(3\rho^2)^{m}}{\abs{\set{P}^m}} + \frac{4\rho}{\abs{\set{P}^m}} \cdot (3\rho^2)^m \label{sensP4}\\
& \leq \frac{5\rho (3\rho^2)^m}{\abs{\set{P}^m}}, \label{sensP}
\end{align}
where~\eqref{sensP1} holds by plugging~\eqref{unif} in~\eqref{eq:final}, \eqref{sensP2} holds since $\abs{\Pset^{m}} \leq \abs{\Pset^{i}}$ for every $i\in [m]$, \eqref{sensP3} holds by summing the geometric sequence, and inequalities~\eqref{sensP4} and~\eqref{sensP} hold since $\rho \geq 1$.
\end{proof}

\setcounter{subsection}{1}
\renewcommand{\thesection}{\Alph{section}}

\subsection{Proof of Theorem~\ref{theorem:coreset}}
\renewcommand{\thesection}{\arabic{subsection}}
\setcounter{subsection}{4}

\epscoreset*

\begin{proof}
\textbf{(i): } Let $J$ denote the number of while iterations in Algorithm~\ref{alg:wrapper}, and for every $j\in [J]$ let $\set{P}^0_{(j)}$, $\set{P}^m_{(j)}$ and $\set{B}^m_{(j)}$ denote respectively the sets $\set{P}^0$, $\set{P}^m$ and $\set{B}^m$ at the $j$th while iteration of Algorithm~\ref{alg:wrapper}.

By Line~\ref{newPi} of Algorithm~\ref{alg:RobustMedSets}, we observe that the output set $\set{P}^m$ is of size $|\set{P}^m| \geq \frac{\abs{\set{P}}}{(bk)^m}$ for some constant $b$, where $\set{P}$ is the input set to the algorithm.
Therefore, the size of $\set{P}^m_j$ returned at Line~\ref{line:callAlg1} of algorithm~\ref{alg:wrapper} in the $j$th while iteration is
\begin{equation} \label{eqSizePm}
\abs{\set{P}^m_{(j)}} \geq \frac{\abs{\set{P}^0_{(j)}}}{(bk)^m}.
\end{equation}
By~\eqref{eqSizePm} and Line~\ref{alg2:L7} of Algorithm~\ref{alg:wrapper}, we obtain that
\begin{equation} \label{eqSizeP0}
\begin{split}
\abs{\set{P}^0_{(j+1)}} & \leq \abs{\set{P}^0_{(j)}}- \abs{\set{P}^m_{(j)}} \leq \abs{\set{P}^0_{(j)}} - \frac{\abs{\set{P}^0_{(j)}}}{(bk)^m} \\
&= \abs{\set{P}^0_{(j)}}\left( 1-\frac{1}{(bk)^m}\right) \\
&= \abs{\set{P}^0_{(1)}}\left( 1-\frac{1}{(bk)^m}\right)^j \\
&= n\left( 1-\frac{1}{(bk)^m}\right)^j,
\end{split}
\end{equation}
where the second derivation is by~\eqref{eqSizePm}.
Combining that $\abs{\set{P}^0_{(J)}} \geq 1$ with~\eqref{eqSizeP0} we conclude that
\begin{equation} \label{eq:J}
J \leq (bk)^m \log{n}.
\end{equation}
Therefore, by Lines~\ref{line:sens} and~\ref{line:endSens} of Algorithm~\ref{alg:wrapper}, the total sensitivity computed at Line~\ref{line:sumSens} of Algorithm~\ref{alg:wrapper} is equal to
\begin{align*}
t &= \sum_{P\in \set{P}} s(P) \leq \sum_{j \in [J]} \left( \sum_{P \in \set{P}^m_{(j)}} \frac{b}{\abs{\set{P}^m_{(j)}}} \right) + O(1) \\
&= \sum_{j \in [J]} b + O(1) = Jb + O(1) \leq (bk)^{m+1}\log{n}.
\end{align*}
By this and Line~\ref{line:randomSample} of Algorithm~\ref{alg:wrapper},
\[
\abs{S} = \frac{(bk)^{m+1}\log{n}}{\eps^2} \left( \log{\left( (bk)^{m+1}\log{n}\right)} d' + \log{\left( \frac{1}{\delta} \right)}\right).
\]
where $d' = O(md^2k^2)$ is the dimension of the sets clustering query space $(\Pset,  \set{X}_k , \Dt)$; see Section~\ref{sec:VC}. By simple derivations we obtain that:
\[
\abs{S} \in O\left(\left(\frac{md\log{n}}{\varepsilon}\right)^2k^{m+4}\right).
\]

\textbf{(ii): }The pair $(\set{P}^m_{(j)},\set{B}^m_{(j)})$ satisfy Lemma~\ref{lem:sens} for every $j\in [J]$. Hence, with an appropriate $b$ (determined from the proof of Lemma~\ref{lem:sens}), for every $P \in \set{P}^m_{(j)}$ the value $s(P)$ defined at Lines~\ref{line:sens} and~\ref{line:endSens} satisfies for every $C\in \M_k$ that
\[
s(P) = \frac{b}{\abs{\set{P}^m_{(j)}}} \geq  \frac{\Dt(P,C)}{\displaystyle{\sum_{Q\in\set{P}^0_{(j)}}}\Dt(Q,C)} \geq \frac{\Dt(P,C)}{\displaystyle{\sum_{Q\in\set{P}}}\Dt(Q,C)}.
\]

By Theorem~\ref{braverman}, a sample $S$ of $\abs{S} \allowdisplaybreaks \leq \allowdisplaybreaks \frac{bt}{\eps^2}\left( \log{(t)} d' + \log\left( \frac{1}{\delta} \right) \right)$ is an $\eps$-coreset for (the sets clustering query space) $(\Pset,\set{X}_k,\Dt)$. Therefore, by Theorem~\ref{braverman}, the pair $(\set{S},v)$ computed at Lines~\ref{line:randomSample}--\ref{line:reweight} satisfies Property (ii) of Theorem~\ref{theorem:coreset}.

\textbf{Computational time. }Consider a call $\RobustMedForSets(\set{P},k)$ to Algorithm~\ref{alg:RobustMedSets} where $\set{P}$ is an $(n,m)$-set. The $i$th iteration of the for loop at Line~\ref{alg:forLine} takes $O\left(n\left(\frac{1}{(4k)}\right)^{i-1}+k^4\right)$ time. Summing over all the $m$ iterations yields a total running time of $O(n+mk^4)$.

Consider the call $\left( \set{P}^m, \set{B}^m \right) := \RobustMedForSets(\set{P}^0,k)$ at Line~\ref{line:callAlg1} of Algorithm~\ref{alg:wrapper}, which dominates the running time of this algorithm. This call is made $J$ times (in each of the $J$ iterations of the while loop). The set $\set{P}^0$ at the $i$th call is of size $s_i = O\left(n\left(1-\frac{1}{(4k)}\right)^{i-1}\right)$. Therefore, the $i$th such call takes $O(s_i+mk^4)$ time. Summing this running time over every $i\in [J]$, where $J \leq (bk)^m \log{n}$ by~\eqref{eq:J}, yields a total running time of
\[
J\cdot mk^4 + n \sum_{i=1}^J \left(1-\frac{1}{(4k)}\right)^{i-1} \in O\left(n\log(n)(bk)^m \right).
\]

\end{proof}

\setcounter{subsection}{4}
\renewcommand{\thesection}{\Alph{section}}

\section{Polynomial Time Approximation Scheme}
The following theorem states that given $n$ polynomials in $d$ (constant number of) variables of constant degree, then the space $\REAL^d$ can be decomposed into a polynomial ($n^d$) number of cells, such that for every $d$ variables $C$ from the cell $\Delta$ the sign sequence of all the polynomials is the same cell.
\begin{theorem} [Theorem 3.4 in~\cite{chazelle1991singly}] \label{theorem:signDecomp}
Let $d$ be a constant and let $\set{F} = \br{\pol_1,\cdots,\pol_n}$ be a set of $n$ multivariate polynomials of constant degree with range $\REAL^d$ and image $\REAL$. It is possible to split $\REAL^d$ into $O\left(n^{2d-2}\right)$ cells $\Delta(\set{F}) = \br{\Delta_i}$, with the property that for every polynomials $\pol_i$ and every cell $\Delta_j$ it holds that $\pol_i$ is either positive, negative, or equal to $0$ on the entire cell $\Delta_j$. This decomposition, including a set of points $A = \br{a_i}$ with $a_i \in \Delta_i$ can be found in time $O\left(n^{2d-1}\log{n}\right)$.
\end{theorem}

\subsection{Proof of Theorem~\ref{theorem:PTAS}}
\renewcommand{\thesection}{\arabic{subsection}}
\setcounter{subsection}{4}

\thmPTAS*

\begin{proof}
What follows is a constructive proof for the theorem. Algorithm~\ref{alg:polynomials} gives a suggested implementation.

Identify $\set{P} = \br{P_1,\cdots,P_n}$ where $P_i = \br{p^i_1,\cdots,p^i_m}$ for every $i\in [n]$.

First we define a set of $n^2m^2k^2$ polynomials as follows.
For every $i,i' \in [n]$, $j,j' \in [m]$, $\ell,\ell' \in [k]$ and vector $x = (x_1^T \mid \cdots \mid x_k^T) \in \REAL^{dk}$ of $dk$ unknowns ($x_1,\cdots,x_k$ are vectors in $\REAL^d$) , let \[
\pol_{i,j,\ell,i',j',\ell'}(x) = \norm{p^i_j - x_\ell}^2 - \norm{p^{i'}_{j'} - x_{\ell'}}^2
\]
be a polynomial in those $dk$ unknowns, of degree at most $2$, and let $\set{F}$ be a set that contains all those polynomials.
Here, each polynomial in $\set{F}$ contains up to $2d$ variables, and $\abs{\set{F}} = n^2m^2k^2$. A polynomial $\pol_{i,j,\ell,i',j',\ell'}(x)$ is positive iff $p^{i'}_{j'}$ is closer to $x_{\ell'}$ than the distance between $p^i_j$ and $x_\ell$.
Therefore, given a possible assignment $x' = ({x'}_1^T \mid \cdots \mid {x'}_k^T) \in \REAL^{dk}$ for the $dk$ unknowns, the vector of sign values of the polynomials in $\set{F}$ when plugging $x'$ corresponds to a clustering of $\set{P}$ into $k$ clusters centered at ${x'}_1^T, \cdots, {x'}_k^T$, and indicates which point in each input $m$-set is the closest to this cluster center, and vice versa, as follows.
Given $x'$, the first cluster $\C_1 \subseteq \REAL^d$ contains all the points $p^i_j$ such that for every $j' \in [m]$ and $\ell' \in [k]$,
\[
\norm{p^i_j - x_1}^2 \leq \norm{p^{i}_{j'} - x_{\ell'}}^2.
\]
Which, by the definition of the polynomials in $\set{F}$, means that for every $j' \in [m]$ and $\ell' \in [k]$,
\[
\sign(\pol_{i,j,1,i,j',\ell'}(x')) = -1.
\]
This enables us to compute the points $\C_1,\cdots,\C_k \subseteq \REAL^d$ of each cluster that are induced by the sign sequence of $\set{F}$ when plugging $x'$.
Given those clusters $\C_1,\cdots,\C_k \subseteq \REAL^d$, we can apply $\oneMean$ to each such cluster $\C_i$ (since $m=k=1$), to obtain, with probability at least $1-\delta$, the optimal point $\hat{c}_i \in \REAL^d$ that minimizes $\sum_{p \in \C_i} \Dt(p,z)$ over every $z\in \REAL^d$, and its cost $cost_i = \sum_{p \in \C_i} \Dt(p,\hat{c}_i)$. The sum $\sum_{i=1}^{k} cost_i$ is the total cost of this clustering option of $\set{P}$.

Since $\oneMean$ is used to compute $k$ centers of $k$ clusters, the probability that $\hat{c}_1,\cdots,\hat{c}_k$ are the optimal centers is at least $1-k\delta$.

By Theorem~\ref{theorem:signDecomp}, we can decompose $\REAL^{dk}$ into $\abs{\Delta(\set{F})} = (nmk)^{O(dk)}$ cells $\br{\Delta_j}$, such that the sign of each polynomial $\pol_i \in \set{F}$ in an entire cell $\Delta_j \in \Delta(\set{F})$ is the same, i.e., the sign sequence of all the polynomials in $\set{F}$ is the same over the entire cell $\Delta'$. Hence, the number of different such sign sequences is at most the number of different cells, which is $(nmk)^{O(dk)}$.

By iterating over every cell $\Delta' \in \Delta(\set{F})$ and taking the sign sequence of the polynomials in $\set{F}$ in this cell, we would have covered all the different sign sequences, which correspond to all the feasible clustering options of $\set{P}$ into $k$ clusters. For each option we can evaluate the total cost as described above, and pick the clustering with the smallest total cost.

The running time of such an algorithm is dominated by the computation of such an arrangement of $\REAL^{dk}$, and by calling $\oneMean$ $\abs{\Delta(\set{F})}$ times; once for each region $\Delta' \in \Delta(\set{F})$. Computing this arrangement takes $nmk^{O(dk)}$ time by Theorem~\ref{theorem:signDecomp} and produces $\abs{\Delta(\set{F})} \in (nmk)^{O(dk)}$ cells. Now it takes $T(n) \cdot (nmk)^{O(dk)}$ total time for the calls to $\oneMean$.
\end{proof}

\setcounter{subsection}{1}
\renewcommand{\thesection}{\Alph{section}}

\subsection{Proof of Corollary~\ref{cor:PTASKmeans}}

\renewcommand{\thesection}{\arabic{subsection}}
\setcounter{subsection}{4}
\setcounter{theorem}{3}

\corPTASKmeans*
\begin{proof}
We will first compute a coreset for the input $\set{P}$ and the given cost function $\Dt$ and query set $\set{X}_k$, and then find the sets-$k$-means for the (weighted) coreset using Theorem~\ref{theorem:PTAS}.

Recall that in this sets-$k$-means problem, $\Dt(P,C) = \min_{p\in P, c\in C}\norm{p-c}^2$ for every $P,C \subseteq \REAL^d$.

Let $(\set{S},v)$ be an output of a call to $\coreset(\Pset,k, \eps, \delta)$; see Algorithm~\ref{alg:wrapper}. Then by Theorem~\ref{theorem:coreset}, $(\set{S},v)$ is an $\varepsilon$-coreset for $(\set{P}, \set{X}_k,\Dt)$ of size $\abs{\set{S}} \in O\left(\left(\frac{mdk\log{n}}{\varepsilon}\right)^2 k^{O(m)} \right)$ with probability at least $1-\delta$ which is computed in $O\left(n\log(n)(bk)^m \right)$ ; see Section~\ref{sec:coresets}.

Let $Q \subseteq \M$ be a set of size $|Q|=n$ and let $u:Q\to [0,\infty)$ be a weights function. Let $\oneMean$ be an algorithm that takes $Q$ and $u$ as input and returns the point $c^* := \frac{\sum_{q\in Q}u(q)\cdot q}{\sum_{q\in Q}u(q)} \in \M$. Observe that $c^*$ minimizes its sum of weighted squared distances to the points of $Q$, i.e.,
\[
\begin{split}
& \sum_{q\in Q} u(q) \Dt(q,c^*) =
\sum_{q\in Q} u(q) \norm{q-c^*}^2\\
& = \min_{c\in \M} \sum_{q\in Q} u(q) \norm{q-c}^2
= \min_{c\in \M}\sum_{q\in Q} u(q) \Dt(q,c).
\end{split}
\]
Furthermore, observe that $c^*$ can be computed in $T(n) = O(n)$ time.

Plugging $\set{P} = \set{S}$, $w=v$,$Q,u,\oneMean,\alpha=1$ and $T(\abs{\set{S}}) = O(\abs{\set{S}})$ in Theorem~\ref{theorem:PTAS} yields that in $(\abs{\set{S}}mk)^{O(dk)} \in \left(\frac{\log{n}}{\varepsilon}dmk^m\right)^{O(dk)}$ time we can compute $\hat{C} \in \set{X}_k$ such that with probability at least $1-k\cdot\delta$,
\begin{equation} \label{eq:reducee}
\sum_{P\in \set{S}}v(P)\cdot\Dt(P,\hat{C})= \min_{C \in \set{X}_k} \sum_{P\in \set{S}}v(P)\cdot\Dt(P,C).
\end{equation}

Hence, the total running time for obtaining $\hat{C}$ is $\left(\frac{\log{n}}{\varepsilon}dmk^m\right)^{O(dk)} + O\left(n\log(n)(bk)^m \right)$.

Corollary~\ref{cor:PTASKmeans} now holds as
\begin{align}
& \sum_{P\in \set{P}}\min_{p\in P, c\in \hat{C}}\norm{p-c}^2 = \sum_{P\in \set{P}}\Dt(P,\hat{C}) \nonumber\\
& \leq \frac{1}{1-\varepsilon} \cdot \sum_{P\in \set{S}}v(P)\cdot\Dt(P,\hat{C}) \label{eq:usecoreset} \\
& \leq (1+2\varepsilon) \cdot \sum_{P\in \set{S}}v(P)\cdot\Dt(P,\hat{C}) \label{eq:useVarEps}\\
& = (1+2\varepsilon)\cdot\min_{C \in \set{X}_k} \sum_{P\in \set{S}}v(P)\cdot\Dt(P,C) \label{eq:useReduction}\\
& \leq (1+2\varepsilon) (1 + \varepsilon) \cdot \min_{C \in \set{X}_k}\sum_{P\in \set{P}} \Dt(P,C) \label{eq:usecoreset2}\\
& \leq (1+4\varepsilon)\cdot \min_{C \in \set{X}_k} \sum_{P\in \set{P}}\Dt(P,C) \label{eq:useVarEps2}\\
& = (1+4\varepsilon)\cdot \min_{C \in \set{X}_k} \sum_{P\in \set{P}} \min_{p\in P, c \in C} \norm{p-c}^2, \nonumber
\end{align}
where~\eqref{eq:usecoreset} and~\eqref{eq:usecoreset2} hold since $(\set{S},v)$ is an $\varepsilon$-coreset for $(\set{P},\set{X}_k,\Dt)$,~\eqref{eq:useVarEps} and~\eqref{eq:useVarEps2} hold since $\eps \leq \frac{1}{2}$ and~\eqref{eq:useReduction} is by~\eqref{eq:reducee}.
\end{proof}

\setcounter{subsection}{2}
\renewcommand{\thesection}{\Alph{section}}

\subsection{Suggested implementation}
In this section we give a suggested implementation for the constructive proof of Theorem~\ref{theorem:PTAS}; see Algorithm~\ref{alg:polynomials}.

\paragraph{Overview of Algorithm~\ref{alg:polynomials}. }Algorithm~\ref{alg:polynomials} gets as input a set $\set{P}$ of $m$-sets, an integer $k\geq 1$, an error parameter $\varepsilon \in (0,1)$ and the probability of failure $\delta \in (0,1)$. The algorithm returns as output a set $\hat{C} \in \set{X}_k$ of $k$ centers that approximate the optimal cost of the $k$-means for set

\begin{algorithm}[!htb]
   \caption{$\PTASAlg(\Pset, w, k, \oneMean)$}
   \label{alg:polynomials}
\begin{algorithmic}[1]
   \STATE {\bfseries Input:} An $(n,m)$-set $\set{P}$, a weights function $w:\set{P} \to [0,\infty)$, a positive integer $k$, and an algorithm $\oneMean$ as in Theorem~\ref{theorem:PTAS}.

    \STATE {\bfseries Output:} A set $\hat{C} \in \displaystyle \argmin_{C \in\set{X}_k} \sum_{P \in \set{P}} w(P)\Dt(P,C)$; \\\quad\quad\quad see Theorem~\ref{theorem:PTAS}.

    \STATE Identify $\set{P} = \br{P_1,\cdots,P_n}$ where $P_i = \br{p^i_1,\cdots,p^i_m}$ for every $i\in [n]$.

    \STATE Define $w'(p) := w(P)$ for every $p \in P$ and $P \in \set{P}$.

    \STATE Let $x = (x_1^T \mid \cdots \mid x_k^T)^T \in \REAL^{dk}$ be a vector of $dk$ unknowns.

    \FOR{every $i,i' \in [n], j,j'\in [m], \ell,\ell' \in [k]$}
    \STATE $\pol_{i,j,\ell,i',j',\ell'} (x)= \norm{p^i_j - x_\ell}^2 - \norm{p^{i'}_{j'} - x_{\ell'}}^2$
    \\\COMMENT{A polynomial of degree $2$ containing up to $2d$ unknowns from $x$. If this polynomial is positive iff $p^{i'}_{j'}$ is closer to $x_{\ell'}$ than the distance between $p^i_j$ and $x_\ell$.}
    \STATE $\set{F}:= \set{F}\cup \br{\pol_{i,j,\ell,i',j',\ell'}(x) }$
    \ENDFOR

    \STATE Compute a decomposition of $\REAL^{dk}$ into cells $\Delta(\set{F}) = \br{\Delta_j}$ as described in Theorem~\ref{theorem:signDecomp}, and let $A$ contain a representative $a \in \Delta'$ from each cell $\Delta' \in \Delta(\set{F})$.



    \STATE $min = \infty$

    \FOR{every $a \in A$}
        \STATE $sum = 0$

        \FOR{every $\ell \in [k]$}
            \STATE
            $\C_\ell := \left\lbrace \mkern-10mu\begin{array}{c!{\vline width 0.6pt}c}
            p^i_j & \begin{array}{c}
            i\in [n], j\in [m] \text{ s.t.}\\ \forall j'\in [m], \ell'\in [k]\\ \sign\left(\pol_{i,j,\ell,i,j',\ell'}(a)\right)=-1
            \end{array} \end{array} \mkern-17mu\right\rbrace$
            \\\COMMENT{The points of cluster number $\ell$ defined by the sign sequence of the cell representative $a \in A$.}

            \STATE $(\hat{c}_\ell, cost_\ell) := \oneMean(\C_\ell,w')$.
            \\\COMMENT{Compute the optimal center ($k=1$) $\hat{c}_\ell$ of the set $\C_\ell \subseteq \REAL^d$ ($m=1$) and its cost $\cost_\ell$, for a given cost function.}

            \STATE $sum = sum + cost_\ell$
        \ENDFOR
        \IF{$sum < min$}
            \STATE $min = sum$
            \STATE $\hat{C} = \br{\hat{c}_1,\cdots,\hat{c}_k}$
        \ENDIF
    \ENDFOR

    \STATE \bfseries{Return} $\hat{C}$
\end{algorithmic}
\end{algorithm}

\setcounter{subsection}{0}
\renewcommand{\thesection}{\Alph{section}}

\section{Robust Median} \label{sec:suppMed}

\subsection{Proof of Lemma~\ref{theorem:robustMed}}

\textbf{Algorithm~\ref{alg:robustMed} overview: }The algorithm relies on the 2 following observations:
(i) To compute a robust approximation of the entire data, it suffices to compute a robust approximation of a randomly sampled subset of this data of sufficient size; see Line~\ref{line:sampleS} of Algorithm~\ref{alg:robustMed} and Lemma~\ref{lem:robustMedViaCoreset},
(ii) If $b$ is a robust approximation of some input set of elements, then by the (weak) triangle inequality for singletons, one of those elements is a constant factor approximation for $b$; see Line~\ref{line:compq} of Algorithm~\ref{alg:robustMed}.

\setcounter{theorem}{0}
\begin{lemma}
\label{lem:robustMedViaCoreset}
Let $\set{P}$ be an $(n,m)$-set, $k \geq 1$, $\delta, \gamma\in (0,1)$, and $\tau \in (0,1/10)$. Pick uniformly, i.i.d, a (multi)-set $\set{S}$ of
\[
|\set{S}| = \frac{c}{\tau^4 \gamma^2} \left( md^2 + \log{\left( \frac{1}{\delta}\right)} \right)
\]
elements from $\set{P}$, where $c$ is a sufficiently large universal constant. Then with probability at least $1-\delta$, any $((1-\tau)\gamma, \tau, 2)$-median of $\set{S}$ is also a $(\gamma, 4\tau, 2)$-median of $\set{P}$.
\end{lemma}

\begin{proof}
For every $P \in \set{P}$ and $b \in \M_1$ define $f_{P}(b) = \Dt(P,b)$. Let $F = \br{f_P \mid P \in \set{P}}$ and $F_\set{S} = \br{f_P \mid P \in \set{S}}$.
Observe that by Definition 4.2 in~\cite{feldman2011unified}, the dimension of the function space $(F,X_1)$ is equivalent to the dimension $d'=md^2$ of the query space $(\Pset,X_1,\Dt)$.
Since $F_\set{S}$ is a random sample of $\frac{c}{\tau^4 \gamma^2} \left( d' +\log{\left( \frac{1}{\delta}\right)}\right) = \frac{c}{\tau^4 \gamma^2} \left( md^2 +\log{\left( \frac{1}{\delta}\right)}\right)  $ functions, sampled i.i.d from $F$, Lemma~\ref{lem:robustMedViaCoreset} now holds by Theorem 9.6 in~\cite{feldman2011unified} which states that a $((1-\tau)\gamma, \tau, 2)$-median of $F_\set{S}$ (which in our case is a $((1-\tau)\gamma, \tau, 2)$-median of $\set{S}$) is a $(\gamma, 4\tau, 2)$-median of $F$ (which in our case is a $(\gamma, 4\tau, 2)$-median of $\set{P}$).
\end{proof}

\setcounter{theorem}{0}

\renewcommand{\thesection}{\arabic{subsection}}
\setcounter{subsection}{5}
\setcounter{theorem}{0}

\robustMed*

\begin{proof}
Let $\gamma = 1/(2k)$ and $\tau = 1/24$.
For a sufficient constant $b$, the random sample $S$ in Line~\ref{line:sampleS} satisfies Lemma~\ref{lem:robustMedViaCoreset}. Therefore, \begin{equation} \label{eq:medSmedP}
\begin{split}
&\text{a }(23/(48k), 1/24, 2)\text{-median of }\set{S} \text{ is also a }\\&(1/(2k), 1/6, 2)\text{-median of }\set{P}.
\end{split}
\end{equation}

Let $q_\set{S}^*$ be the $(23/(48k),0,0)$-median of $\set{S}$, and let $q_\set{\set{S}}'$ be the closest point in $\set{S}$ to $q_\set{S}^*$, i.e.,
\[
q_\set{\set{S}}' \in \argmin_{q \in Q: Q \in \set{S}} \Dt(q_\set{\set{S}}^*,q).
\]
By the weak triangle inequality from Lemma~\ref{lem:weakTriangle}, we have that $\Dt(P,q_\set{\set{S}}') \leq 2\rho\Dt(P,q_\set{\set{S}}^*)$ for every $P \in \set{S}$, i.e., that $q_\set{\set{S}}'$ is a $2$-approximation for $q_\set{\set{S}}^*$.
This yields that $q_\set{S}'$ is a $(23/(48k),0,2)$-median of $\set{S}$, which is also a $(23/(48k),1/6,2)$-median of $\set{S}$. Hence, one of the points of $\set{S}$ is a $(23/(48k),1/6,2)$-median of $\set{S}$.
Therefore, the point $q$ computed at Line~\ref{line:compq} and returned in Line~\ref{line:retMed} is such a $(23/(48k),1/6,2)$-median of $\set{S}$, which by~\eqref{eq:medSmedP} is also a $(1/(2k), 1/6, 2)$-median of $\set{P}$.

The computation time of Algorithm~\ref{alg:robustMed} is dominated by Line~\ref{line:compq}, which can be implemented in $t|\set{S}|^2 = tb^2k^4\log^2\left(\frac{1}{\delta}\right)$ time by simply computing the pairwise distances between every two sets in $\set{S}$ and using order statistics.
\end{proof}

\setcounter{subsection}{1}
\renewcommand{\thesection}{\Alph{section}}

\section{Implemented Algorithms}\label{IA}

\textbf{$\exact(\set{P})$} is implemented by what we call \emph{sets Voronoi diagram}; see Fig.~\ref{fig:setsVoronoi}.

\begin{figure}
  \centering
  \includegraphics[width=0.35\textwidth]{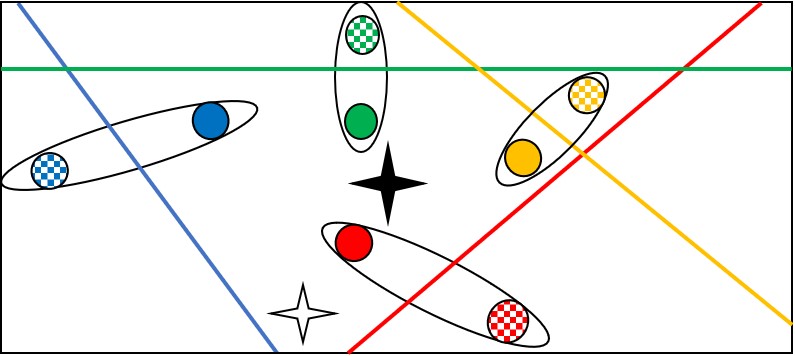}
  \caption{\textbf{$\exact$ via sets Voronoi diagram. }A set of $n=4$ pairs on the plane ($m=d=2$) and its sets Voronoi diagram which is computed as follows: (i) A set voronoi diagram is computed for each pair ($m$-set) to obtain a set of hyperplanes, (ii) an arrangement of those hyperplanes is then computed, which results in a partition of $\REAL^2$ into the cells which are presented above. Each cell corresponds to a selection of representatives, one from each pair. The sets-mean $c^*$ (solid star) is also the $1$-mean of the representative points shown in solid circles, which correspond to this Voronoi cell. Any other point (empty star) inside the same Voronoi cell as $c^*$ admits the same set of representatives. Therefore, to compute $c^*$, it suffices to exhaustive search over all the Voronoi.}
  \label{fig:setsVoronoi}
\end{figure}

\textbf{$\km(\set{P},k)$.} We focused on the sets-$k$-means case (see Section~\ref{sec:intro} and Table~\ref{table:distFuncs}), where the clustering algorithm we applied is a modified version of the the well know Lloyd algorithm~\cite{lloyd1982least} as follows.
The algorithm starts by an initial $k$ random centers $C \subseteq \br{p\in P \mid P \in \set{P}}$. It then assigns every $P \in \set{P}$ to its closest center $c_P = \argmin_{c\in C} \Dt(P,c)$.
Finally, it replaces every $c\in C$ with the sets-mean of the (possibly weighted) sets $\br{P \in \set{P} \mid c_P = c}$ in its cluster.
It repeats this process till convergence, but no more than $12$ iterations. The sets-mean is computed as follows.

\textbf{$\approxalg(\set{P},t)$. }As explained in Section~\ref{sec:intro}, computing the sets-mean $c^*$ is a non-trivial and time consuming task. However, at least $\abs{\set{P}}/2$ of the input sets $P \in \set{P}$ satisfy that $\Dt(P,c^*) \leq \frac{2\sum_{Q \in \set{P} \Dt(Q,c^*)}}{n}$. By the triangle inequality for singletons (Lemma~\ref{lem:weakTriangle}), it follows immediately that the closest point $p\in P$ to $c^*$ is a $3$-approximation for $c^*$. Therefore, with probability at least $1/2$, one of the points of a randomly sampled input set is a good approximation. We can amplify this probability by sampling $t\geq 1$ such sets.


\textbf{Handling sets of different sizes. }
For example in dataset (ii), each newspaper $P_i$ consists of different number of paragraphs and hence is represented by a different number $|P_i|$ of vectors. Let $z$ denote the maximal such set size. To compute a coreset for such dataset $\set{P}$, we first partition $\set{P}$ into $z$ sets $\set{P} = \set{P}_1 \cup \cdots \cup \set{P}_z$ where $\set{P}_i$ contains all the sets $P\in \set{P}$ of size $|P| = i$. Then, for every $i \in [z]$, we plug $\set{P}^0 = \set{P}_i$ at Lines~\ref{line:P0}--~\ref{line:endSens} of Algorithm~\ref{alg:wrapper} to compute $s(P)$ for every $P\in \set{P}_i$. In other words, we compute the sensitivity bound for each set on its own. We compute the total sensitivity $t_i := \sum_{P \in \set{P}_i}s_i(P)$ of each set $\set{P}_i$ and $t:=\sum_{i\in [m]} t_i$ to be their total. We then simply perform Lines~\ref{line:randomSample}--~\ref{line:ret}.


\end{document}